\documentclass{article}
\usepackage{arxiv}

\usepackage[utf8]{inputenc} % allow utf-8 input
\usepackage[T1]{fontenc}    % use 8-bit T1 fonts
\usepackage{hyperref}       % hyperlinks
\usepackage{url}            % simple URL typesetting
\usepackage{booktabs}       % professional-quality tables
\usepackage{amsfonts}       % blackboard math symbols
\usepackage{nicefrac}       % compact symbols for 1/2, etc.
\usepackage{microtype}      % microtypography
\usepackage[svgnames]{xcolor}         % colors
\usepackage{tikz}
\usetikzlibrary{calc,fit}

\usepackage[small]{caption}
\usepackage{algorithm}
\usepackage[noend]{algpseudocode}
\usepackage{amssymb}
\usepackage{mathtools}
\usepackage{amsmath}
\usepackage{bbm}
\usepackage{tikz}
\usetikzlibrary{fit, calc}
\usepackage{todonotes}
\usepackage{enumitem}
\usepackage{tikz}
\usepackage{amsthm}
\usepackage{thm-restate}
\usepackage{booktabs}
\usepackage{multirow}
\usepackage{newfloat}
\usepackage{wrapfig}
\usepackage{comment}
\usepackage{multirow}
\usepackage{xspace}
\usepackage{subfigure}
\usetikzlibrary{decorations.markings} % if not already loaded
\usepackage{mdframed}           % very light-weight boxed environment
\usepackage{natbib}
\usepackage{environ}
\definecolor{lightblue}{rgb}{0.8, 0.9, 1}
\definecolor{lightred}{rgb}{1, 0.8, 0.8}
\definecolor{lightgreen}{rgb}{0.8, 1, 0.8}

\newtheorem{lemma}{Lemma}
\newtheorem{remark}{Remark}
\newtheorem{corollary}{Corollary}

\newtheorem{assumption}{Assumption}

\newcommand{\xvec}{\boldsymbol{x}}
\newcommand{\lvec}{\boldsymbol{\ell}}
\newcommand{\gvec}{\boldsymbol{g}}
\newcommand{\uvec}{\boldsymbol{u}}

\DeclareMathOperator*{\argmax}{arg\,max}

\DeclareMathOperator*{\argmin}{arg\,min}

\newcommand{\algnamesoft}{\textnormal{\texttt{COLB}}\@\xspace}
\newcommand{\algnamehard}{\textnormal{\texttt{SOLB}}\@\xspace}

\allowdisplaybreaks

\author{
        Gianmarco Genalti\thanks{Equal contribution.}\\
        Politecnico di Milano\\
	\texttt{gianmarco.genalti@polimi.it}
	\And
	Francesco Emanuele Stradi\footnotemark[1]\\
	Politecnico di Milano\\
	\texttt{francescoemanuele.stradi@polimi.it}
	\AND
	 Matteo Castiglioni\\
	Politecnico di Milano\\
	\texttt{matteo.castiglioni@polimi.it}
	\And
	Alberto Marchesi\\
	Politecnico di Milano\\
	\texttt{alberto.marchesi@polimi.it}
	\And
	 Nicola Gatti\\
	Politecnico di Milano\\
	\texttt{nicola.gatti@polimi.it}
}

\title{Data-Dependent Regret Bounds for Constrained MABs}
\begin{document}

\maketitle
\begin{abstract}
This paper initiates the study of \emph{data-dependent} regret bounds in \emph{constrained} MAB settings.
These bounds depend on the sequence of losses that characterize the problem instance. 
Thus, they can be much smaller than classical $\widetilde{\mathcal{O}}(\sqrt{T})$ regret bounds, while being equivalent to them in the worst case.
Despite this, data-dependent regret bounds have been completely overlooked in constrained MAB settings.
The goal of this paper is to answer the following question: \emph{Can data-dependent regret bounds be derived in the presence of constraints?}
We answer this question affirmatively in constrained MABs with {adversarial} losses and {stochastic} constraints.
Specifically, our main focus is on the most challenging and natural settings with \emph{hard constraints}, where the learner must ensure that the constraints are always satisfied with high probability.
We design an algorithm with a regret bound consisting of \emph{two} data-dependent terms.
The first term captures the difficulty of satisfying the constraints, while the second one encodes the complexity of learning independently of the presence of constraints.
We also prove a lower bound showing that these two terms are \emph{not} artifacts of our specific approach and analysis, but rather the fundamental components that inherently characterize the complexities of the problem.
Finally, in designing our algorithm, we also derive some novel results in the related (and easier) \emph{soft constraints} settings, which may be of independent interest.
\end{abstract}

\section{Introduction}
Over the past few years, constrained \emph{multi-armed bandit} (MAB) problems have gained increasing popularity in learning theory (see, \emph{e.g.},~\citep{liakopoulos2019cautious,pacchiano,castiglioni2022online,chen2022strategies}).
In unconstrained MAB problems, the learner is evaluated solely in terms of \emph{regret}, which measures the difference between the learner's performance and the performance of an \emph{optimal-in-hindsight} decision.
Constrained settings introduce additional challenges, as the learner must also ensure that certain constraints are \emph{not} violated excessively while learning.

A growing trend in unconstrained MAB research is the derivation of \emph{data-dependent} regret bounds (see, \emph{e.g.},~\citep{neu2015explore,lee2020bias}).
These bounds depend on the sequence of losses that characterize the problem instance.
Some example of such bounds---typically called \emph{small-loss} regret bounds---are of the form $\widetilde{\mathcal{O}}(\sqrt{L^*})$, where $L^*$ is the cumulative loss of an optimal-in-hindsight decision.
Clearly, data-dependent bounds can be much smaller than classical $\widetilde{\mathcal{O}}(\sqrt{T})$ regret bounds, while being equivalent to them in the worst case.
Despite this, data-dependent regret bounds have been completely overlooked in the literature on constrained online learning problems.

The main goal of this paper is to initiate the study of data-dependent regret bounds in constrained MAB settings.
In particular, we aim at answering the following research question:
\begin{center}
    \emph{Can data-dependent regret bounds be derived in the presence of constraints?}
\end{center}
In this paper, we provide an affirmative answer to the question above, as described in the following.

Due to space constraints, we refer to Appendix~\ref{app:related} for the complete discussion on related works.

\subsection{Original Contributions}

Given an impossibility result by~\citet{Mannor}, which prevents from obtaining sub-linear regret in constrained settings where \emph{both} the losses and the constraints are selected \emph{adversarially}, in this paper we focus on constrained MABs with \emph{adversarial} losses and \emph{stochastic} constraints, as customarily done in the literature (see, \emph{e.g.},~\citep{Upper_Confidence_Primal_Dual}).
Specifically, our main focus is on the most challenging and natural settings with \emph{hard constraints}, where the learner's goal is to minimize the regret while ensuring that the constraints are satisfied at every round with high probability.

\paragraph{Warm-Up: Soft Constraints}
As a preliminary step toward our final goal, we design an algorithm with a \emph{small-loss} regret bound in constrained MAB settings with \emph{soft constraints}, which may be of independent interest.
These settings only require that the \emph{constraint violations} grow sub-linearly with the number of rounds $T$, thereby allowing to violate the constraints in some rounds.
Although soft constraints are technically easier than hard ones, our algorithm incorporates some key components that are also central to the algorithm for hard constraints.
%
% This makes the analysis of settings with soft constraints of independent interest.
%
At a high level, our algorithm builds upon an approach introduced by~\citet{lee2020bias} to derive small-loss regret bounds in unconstrained problems.
Our algorithm---called \emph{Constrained OMD with Log-Barrier} (\algnamesoft)---applies this approach to a ``safe'' set of decisions that optimistically satisfy the constraints.  
This allows to attain sub-linear violations with a high-probability regret bound of $\widetilde{\mathcal{O}}( \sqrt{\sum_{t=1}^{T}\lvec_t^\top\xvec^*} )$, where $\xvec^*$ is a decision that is optimal in hindsight while satisfying the constraints in expectation.
\paragraph{Hard Constraints}
Our main result in this paper is an algorithm with a data-dependent regret bound in constrained MAB settings with \emph{hard constraints}.
Our algorithm for these settings---called \emph{Safe OMD with Log-Barrier} (\algnamehard)---employs all the core components of \algnamesoft, and it adds new ones to deal with hard constraints.
Specifically, \algnamehard always chooses a suitable combination between the decision suggested by \algnamesoft and a decision that \emph{strictly} satisfies the constraints, which is given as input to the algorithm.
The \algnamehard algorithm achieves a high-probability regret bound that is characterized by two data-dependent terms, described in the following. %\stradi{The first term is $\widetilde{\mathcal{O}}(\sqrt{\sum_{t=1}^T (\lvec_t^\top ( \xvec^\diamond - \xvec^* ))^2 })$, where $\xvec^\diamond$ is the decision that strictly satisfies the constraints given as input to the algorithm. This term intuitively captures the difficulty of satisfying the constraints at every round. The second term is $\widetilde{\mathcal{O}}( \sqrt{\sum_{t=1}^{T}\lvec_t^\top\xvec^*} )$, and it intuitively encodes the performance of a decision $\xvec^*$ that is optimal in hindsight while satisfying the constraints in expectation.}
The first term is $\widetilde{\mathcal{O}}(\sqrt{\sum_{t=1}^T (\lvec_t^\top ( \xvec^\diamond - \xvec^* ))^2 })$, where $\xvec^\diamond$ is the decision that strictly satisfies the constraints given as input. 
     This term captures the difficulty due to constraints.
The second term is $\widetilde{\mathcal{O}}( \sqrt{\sum_{t=1}^{T}\lvec_t^\top\xvec^*} )$, and it intuitively encodes the performance of a decision $\xvec^*$ that is optimal in hindsight while satisfying the constraints in expectation.
Finally, we also provide a lower bound that demonstrates that the two terms in the regret bound of the \algnamehard algorithm are \emph{not} artifacts of our specific approach and analysis, but rather the fundamental components that inherently characterize the complexities of the problem.
Interestingly, this result also shows that data-dependent regret bounds can not only outperform classical $\widetilde{\mathcal{O}}(\sqrt{T})$ bounds, but also offer insights into the underlying complexities of learning problems.
%\begin{itemize}
    %\item The goal of this paper is to initiate the study of small-loss regret bounds in the constrained MAB framework.
    %
    %\item Given the impossibility result by~\citet{Mannor}, we focus on settings in which the losses are adversarial while the costs characterizing the constraints are stochastic.
    %
    %\item Our main result is a characterization of small-loss regret bounds attainable in the most challenging setting with \emph{hard constraints}. In this setting, the learner's goal is to minimize the regret while satisfying the constraints at every round with high probability.
    %
    %\item In order to derive our main result, we also address a \emph{soft constraints} setting, as a preliminary step.
    %
    %\item Our results unveil a subtle, yet natural, structure in the complexity of the regret-minimization problem in the hard constraints settings. Moreover, we show that this structure is \emph{not} a mere outcome of our specific analysis, but it is inherent of the problem, as shown by a small-loss lower bound.
    %
    %Description of the algorithms and the most relevant techniques \ldots
%\end{itemize}

%\section{Problem Definition}
%\label{sec:setting}
%In this section, we introduce the preliminary concepts and definitions needed in the rest of the paper. 
\section{Constrained Multi-Armed Bandits}
\label{sec:setting}
% \begin{wrapfigure}[7]{R}{0.48\textwidth}
% 	\vspace{-0.8cm}
% 	\begin{minipage}
% 		{0.48\textwidth}
% 		\begin{algorithm}[H]
% 			\RestyleAlgo{ruled}
% 			\caption{Learner-Environment Interaction}
% 			\label{alg: Learner-Environment Interaction}
% 			\DontPrintSemicolon
% 			\small
% 			\For{$t	\in[T]$} {
% 				Environment chooses $\lvec_t$ and $\gvec_{t,i}$ for $i\in[m]$
				
% 				Learner chooses a strategy $\xvec_{t} \in\Delta_K$ 
				
% 				Learner plays an action $a_t \sim \xvec_{t}$ 
				
% 				Learner observes $\ell_{t}(a_t)$ and $g_{t,i}(a_t)$ for $i\in[m]$ 
				
% 			}
% 		\end{algorithm}
% 	\end{minipage}
% \end{wrapfigure}
In the \emph{multi-armed bandit} (MAB) framework~\citep{lattimore2020bandit}, a learner is repeatedly faced with a decision among $K  \in \mathbb{N}_+$ actions over $T \in \mathbb{N}_+$ rounds.
At each round $t\in[T]$,\footnote{Given $n\in\mathbb{N}$, we denote by $[n]$ the set $\{1,\dots,n\}$ of the first $n$ natural numbers.} the learner chooses a strategy (\emph{i.e.}, a probability distribution over actions) $\xvec_t \in \Delta_K$, where $\Delta_K$ is the simplex of dimension $K-1$.
Then, they play an action $a_t \sim \xvec_t$ sampled according to this strategy and observe a loss $\ell_t(a_t)$, which is defined by a vector $\lvec_t\in [0,1]^{K}$ of losses at round $t$.\footnote{In this paper, we use the notation $c(i)$ to denote the $i$-th element of vector $\boldsymbol{c}$.}
%
% We denote the $a$-th element of $\xvec_t$, which encodes the probability of playing action $a \in [K]$, as $x_t(a)$.
%
%\footnote{In general, in this paper we use the notation $c(i)$ to denote the $i$-th element of vector $\boldsymbol{c}$.}
%
% The vector of losses $\lvec_t\in [0,1]^{K}$, with $\ell_t(a)$ encoding the loss of action $a \in [K]$, may be chosen by an \textit{adaptive} adversary that is aware of the history of interaction up to round $t-1$.
%
% Thus, no statistical assumption is made on the process generating the losses.
%
%At each time $t \in[T]$, only the $a_t$-component $\ell_t(a_t)$ of the loss vector $\lvec_t$ is revealed to the learner. 
%Moreover, losses are assumed to be bounded in $[0,1]$.
In this paper, we study a \emph{constrained} version of the MAB framework~\citep{pacchiano}.
At each $t\in[T]$, in addition to a loss, the learner also observes $m \in \mathbb{N}_+$ constraint costs $g_{t,i}(a_t)$, one for each constraint $i\in[m]$.
Each of them is determined by a vector $\gvec_{t,i}\in[0,1]^{K}$ of constraint costs at round $t$.
Each constraint $i \in [m]$ is associated to a threshold $\alpha_i \in [0,1]$, and it is considered satisfied by a learner's strategy whenever the constraint cost is below $\alpha_i$ in expectation.

Motivated by a well-known impossibility result by~\citet{Mannor},\footnote{\citet{Mannor} show that if \emph{both} the losses and the costs are selected adversarially, it \emph{not} possible to design an algorithm that simultaneously achieves sublinear (in $T$) regret and sublinear (in $T$) cumulative constraint violation.} in this paper we focus on constrained MAB problems in which the losses are chosen \emph{adversarially} and the constraint costs are selected \emph{stochastically}.
Specifically, we assume that, at each round $t \in [T]$, the loss vector $\lvec_t$ is chosen by an \emph{adaptive} adversary that is aware of the history of interaction up to round $t-1$, while each cost vector $\gvec_{t,i}$, for $i \in [m]$, is sampled independently from a probability distribution $\mathcal{G}_i$.
For ease of notation, in the following we use $\lvec_{1:T}$ to denote the sequence of all loss vectors $\lvec_t$, for $t \in [T]$, while we denote by $\gvec_i \coloneqq \mathbb{E}_{\gvec \sim \mathcal{G}_i} [\gvec]$ the expected value of $\mathcal{G}_i$, for every $i \in [m]$.

The performance (in terms of losses) of a learning algorithm is usually measured in terms of regret with respect to a \emph{baseline}.
In the constrained MABs addressed in this paper, the baseline is formally defined by the following optimization problem parametrized by $\lvec_{1:T}$ and $\gvec_i$ for $i \in [m]$:
%
% In order to define the performance measures used to evaluate our learning algorithms, we first need to introduce an \emph{offline} optimization problem.
%
% This is parametrized by a given loss vector $\lvec \in [0,1]^K$ and collection $\{ \gvec_{i} \}_{i \in [m]}$ of constraint cost vectors $\gvec_i \in [0,1]^K$, and it is defined as follows:
%
\begin{equation}
	\label{lp:offline_opt}\text{OPT}({ \lvec_{1:T}, \{\gvec_i\}_{i\in[m]}}):=\begin{cases}
		\displaystyle\min_{\xvec \in \Delta_K} &  \sum_{t =1}^T \lvec_t^\top \xvec
        %\lvec^\top \xvec
		\quad \textnormal{s.t.} \\
		& \gvec_i^\top\xvec\leq  \alpha_i \quad \forall i\in[m].
	\end{cases}
\end{equation}
Program~\eqref{lp:offline_opt} encodes the value of a strategy that is \emph{optimal in hindsight}, \emph{i.e.}, a strategy that minimizes the cumulative loss while ensuring that the constraints are satisfied in expectation.
In the following, we denote one such strategy, which is an optimal solution to Program~\eqref{lp:offline_opt}, as $\xvec^* \in  \Delta_K$.
%
%
% \subsection{Performance Measures}
%
%
Then, for a sequence of losses $\lvec_{1:T}$, the \emph{(cumulative) regret} over the $T$ rounds is defined as
%
%
%In the constrained MAB framework, the performances of a learning algorithm are evaluated in terms of regret and constraint violations.
%
%The \emph{cumulative regret} incurred by a learning algorithm is:
%\gianmarco{mettere $R_T(\ell_{1:T})$, la definizione di regret è comodo averla che dipende dall'istanza}
 $
	R_T(\lvec_{1:T}) \coloneqq \sum_{t=1}^T \lvec_t^\top \xvec_t- \text{OPT} (\lvec, \{\gvec_i\}_{i\in[m]}) = \sum_{t=1}^T \lvec_t^\top \xvec_t-\sum_{t=1}^T\lvec_t^\top \xvec^*.
$
% where we let $\lvec\coloneqq \frac{1}{T}\sum_{t=1}^T \lvec_t$ and ${\gvec}_i\coloneqq \mathbb{E}_{\gvec\sim\mathcal{G}_i}[\gvec]$ for $i \in [m]$ in the offline optimization problem defining the baseline against which algorithm losses are compared.
%
% Notice that the regret is evaluated with respect to an \emph{optimal-in-hindsight strategy that satisfies the constraints in expectation}, which is an optimal solution to Program~\eqref{lp:offline_opt}.
%
% In the following, we denote by $\xvec^* \in \Delta_K$ one such strategy, so that the regret definition reduces to $	R_T\coloneqq \sum_{t=1}^T \lvec_t^\top \xvec_t-\sum_{t=1}^T\lvec_t^\top \xvec^*.$  

\begin{remark}[On the regret definition]
	Differently from the standard MAB framework~\citep{lattimore2020bandit}, in constrained settings an optimal strategy $\xvec^*$ may {not} coincide with a vertex of the simplex $\Delta_K$, i.e., an optimal action may {not} exist.
	This is intuitive since, whenever the action associated with the smallest loss in hindsight does not satisfy the constraints in expectation, an optimal strategy $\xvec^*$ may play that action as much as possible, while satisfying the constraints in expectation.
	This is the reason why the regret $R_T(\lvec_{1:T})$ is defined with respect to (randomized) strategies, rather than actions.
	This definition is standard in constrained online learning settings (see, e.g.,~\citep{Exploration_Exploitation,pacchiano,safetree}).
\end{remark}

In this paper, our goal is to design learning algorithms for constrained MAB settings that achieve \emph{small-loss style} regret bounds of the form $R_T(\lvec_{1:T})\leq\widetilde{\mathcal{O}} ( \sqrt{\sum_{t=1}^T \lvec_t^\top \xvec^*} )$, where $\sum_{t=1}^T \lvec_t^\top \xvec^*$ represents the cumulative loss incurred by a strategy $\xvec^*$ that is optimal in hindsight.
These regret bounds are arbitrarily better than common $\widetilde{\mathcal{O}}(\sqrt{T})$ regret bounds when $\xvec^*$ outperforms other strategies on the sequence of losses $\lvec_{1:T}$, while being equivalent to them in the worst case.
It is well known that small loss regret bounds can be achieved in unconstrained MABs (see, \emph{e.g.},~\citep{lee2020bias}), but it remains an open question whether they can also be derived in constrained settings.

In constrained MABs, a fundamental challenge is the fact that the learner must account for constraint violations during learning.
Our primary focus is on satisfying the constraints at every round with high probability.
However, we also derive some results for the weaker goal of minimizing cumulative violations, as a preliminary step.
%
% These two goals are typically referred to in the literature as the \emph{hard constraints} and \emph{soft constraints} settings, respectively.
%
Next, we formally introduce these two goals.

\paragraph{Soft Constraints}
In this setting, the goal of the learner is to minimize the \emph{(cumulative) positive constraint violations} over the $T$ rounds, defined as
$V_T\coloneqq \max_{i\in[m]} \sum_{t=1}^T \left[\gvec_{i}^\top \xvec_{t}-\alpha_i\right]^+,$  where we let  $ \left[\cdot\right]^+\coloneqq\max\{0,\cdot\}.$
Intuitively, $V_T$ represents the total constraint violation accumulated by the learner during the learning process, and it also ensures that negative violations (\emph{i.e.}, constraint satisfactions) do \emph{not} cancel out positive ones.
The goal is to guarantee that $V_T = o(T)$.\footnote{Let us remark that, in the soft constraints setting, we are interested in obtaining small-loss bounds \emph{only for the regret}. Indeed, small-loss bounds naturally belong to adversarial settings, as they become vacuous in stochastic ones (recall that constraint costs are stochastic in our setting). Moreover, an optimal-in-hindsight strategy $\xvec^*$ satisfies the constraints in expectation, by definition. Thus, it would \emph{not} make any sense to derive a small-loss bound for $V_T$.}

\paragraph{Hard Constraints}
In this setting, the goal of the learner is to guarantee that $\gvec_i^\top \xvec_t \leq \alpha_i$ for every constraint $i \in [m]$ and round $t \in [T]$ with high probability.\footnote{A constraint $i \in [m]$ is satisfied at round $t \in [T]$ whenever its cost is below $\alpha_i$ in expectation over the randomness of the strategy $\xvec_t$ and the cost $g_{t,i}$. This is standard in constrained MABs (\emph{e.g.}~\citep{pacchiano}).}
This objective is only attainable under the following two assumptions, which are common in the literature on hard constraints settings (see, \emph{e.g.},~\citep{pacchiano,bounded,safetree}).
\begin{assumption}[Slater's condition]\label{ass:Slater}
    Program~\eqref{lp:offline_opt} satisfies Slater's condition, i.e., it admits a strictly feasible solution $ \xvec^{\diamond} \in \Delta_K$, which is a strategy such that $\gvec_i^\top \xvec^\diamond<  \alpha_i$ for every constraint $i\in[m]$.
\end{assumption}
\begin{assumption}[Knowledge of a strictly feasible strategy]\label{ass:knowldege}
    The learner knows a strictly feasible strategy $\xvec^\diamond \coloneqq\argmax_{\xvec\in\Delta_K}\min_{i \in [m]} [\alpha_i-\gvec_i^\top\xvec]$ and its associated cost $\theta_i\coloneqq\gvec_{i}^\top\xvec^\diamond$.\footnote{Previous works (see, \emph{e.g.},~\citep{pacchiano,bounded,safetree}) usually assume to know a generic strictly feasible strategy. For ease of presentation, we assume to know a strategy that satisfies the constraints as much as possible. Nonetheless, our results can be generalized to the case commonly considered in the literature.}
    We denote by $\rho \in [0,1]$ to margin by which $\xvec^\diamond$ satisfies the constraints, formally
    $\rho\coloneqq \min_{i\in[m]} \left[\alpha_i-\theta_i\right]$.
\end{assumption}
Intuitively, Assumptions~\ref{ass:Slater}~and~~\ref{ass:knowldege} are necessary to ensure that, in early rounds when little information is available, the learner has sufficient exploration opportunities without violating the constraints.

\section{Being Safe While Learning With an Increasing Learning Rate}
\label{sec:safe_playground}

We begin by presenting two core components of our algorithms presented in Sections~\ref{sec:soft_main}~and~\ref{sec:hard_main}.
The first one is a \textit{safe decision space}, which is a restricted set of strategies used to control constraint violations.
The second component is \textit{OMD with log-barrier}~\citep{lee2020bias}, which is an existing algorithm that achieves data-dependent regret bounds in unconstrained settings and serves as a foundation for our algorithms.
The goal of this section is to show how these two components can be combined to achieve data-dependent regret bounds while controlling constraint violations.

\subsection{Costs Estimation and Safe Decision Spaces}
\label{sec:cost_estimation}

Estimating costs is a crucial task for any algorithm operating in constrained MABs. However, using these estimates in order to control constraint violations may \emph{not} be trivial.
%
%to play according to a certain safety criterion may not be trivial.
%
Next, we describe the approach used by our algorithms.
%
%procedure to estimate expected costs and use such estimates to control our algorithm's violations.
%
In the following, we let $N_t(a)\coloneqq \sum_{\tau=1}^t \mathbbm{1}_{\{a_\tau = a\}}$ be the number of rounds up to $t \in [T]$ in which action $a\in[K]$ is played. Then, an unbiased estimator for the cost of constraint $i\in[m]$ when playing action $a \in [K]$ is 
$
  \widehat{g}_{t,i}(a) \coloneqq \frac{1}{\max\{1,N_t(a)\}}\sum_{\tau\in[t]}g_{\tau,i}(a)\mathbbm{1}_{\{a_\tau = a\}}. 
$
% The estimator presented the equation above represents the sample mean of the observed costs up to a given time $t$.
%
The following result quantifies the uncertainty associated with the estimator above.
\begin{lemma}
\label{lemma:global_hoeffding}
Let $\delta \in (0,1)$ and $\beta_t(a, \delta):=\min \left\{1,\sqrt{{4\ln\left(\nicefrac{TKm}{\delta}\right)}/{\max\{1,N_t(a)\}}} \right\}$ for $a \in [K]$. Then, with probability at least $1-\delta$, $
    \left|\widehat g_{t,i}(a)- g_{i}(a)\right|\leq \beta_{t}(a, \delta)
$ for every $t \in [T], i \in [m], a \in [K]$. 
\end{lemma}
This above lemma is a trivial consequence of Hoeffding's inequality and a union bound. %Results of this form are common in stochastic bandits, as they allow for uncertainty quantification of the estimation. 
We denote by $\widehat\gvec_{t,i} \in [0,1]^{K}$ the vector whose entries are the estimates $g_{t,i}(a)$, while we let $\boldsymbol\beta_t(\delta) \in [0,1]^{K}$ be the vector of the bounds $\beta_t(a, \delta)$. For clarity, in the rest of the paper we omit $\delta$ from the argument of $\boldsymbol{\beta}_t(\delta)$. Moreover, we denote by $\mathcal{E}(\delta)$ the event defined by Lemma~\ref{lemma:global_hoeffding}, which satisfies $\mathbb{P}(\mathcal{E}(\delta))\ge 1-\delta$.

\paragraph{Safe Decision Space} 
Lemma~\ref{lemma:global_hoeffding} provides some confidence intervals for the estimated costs. Next, we describe how these intervals can be used to define a sequence of sets, called \textit{safe decision spaces}, which contain strategies that an algorithm can employ to control the cumulative constraint violations. Formally, for every round $t \in [T]$, we let
$
    \mathcal{S}_t \coloneqq \left\{\xvec\in \Delta_K : { (\widehat{\gvec}_{t,i}-\boldsymbol{\beta}_t)^\top \xvec\leq\alpha_i \,\,\, \forall i\in[m]}\right\}
$
be the \textit{safe decision space} at time $t$. This definition is standard in the constrained MAB literature. Intuitively, an algorithm that selects strategies from $\mathcal{S}_t$ ensures that, with probability at least $1-\delta$, the expected incurred costs remain below the thresholds at every round $t \in [T]$, and for every action $a \in [K]$ and constraint $i \in [m]$. This holds thanks to Lemma~\ref{lemma:global_hoeffding} and the way in which the safe decision spaces are constructed.

\subsection{Data-Dependent Bounds via Increasing Learning Rate}

Next, we recall some needed details of the \textit{OMD with log-barrier} algorithm by~\citet{lee2020bias}.
This algorithm achieves small-loss guarantees in unconstrained MABs, with a regret of $\widetilde{\mathcal{O}}(\sqrt{L^*})$ with high probability, where $L^*$ is the cumulative loss of an action that is optimal in hindsight. The algorithm works as standard OMD with log-barrier regularization $\psi_t(\xvec)=\sum_{a=1}^K\frac{1}{\eta_{t,a}}\ln\frac{1}{x(a)}$, where $\eta_{t,a}$ is an \textit{increasing} sequence of learning rates. Specifically, each time the probability of selecting an action goes below a certain threshold, the learning rate is increased by a constant factor, while the threshold is increased. This procedure is key to achieve the desired data-dependent regret bound. One of the main technical features of OMD with log-barrier is its \emph{restricted decision space}. Indeed, differently from most of the OMD-like algorithms, it is only allowed to select strategies belonging to the \textit{truncated simplex}, which is defined as
$
    \Omega \coloneqq \left\{\xvec\in \Delta_K : { x(a)\geq \frac 1 T \,\,\,  \forall a\in[K] }\right\}.
$
This design choice avoids forced uniform exploration, and remarkably simplifies the analysis.

\subsection{A Truncated Safe Decision Space}
\label{sec:truncated_safe}

Our algorithms, presented in the following sections, rely on combining an OMD with log-barrier sub-routine with safe decision-making.
%
%The algorithms that will be presented in the next sections rely on an OMD with log-barrier sub-routine combined with safe decision-making.
%
This raises some challenges, since both components put some restrictions on the space from which strategies are chosen.
%
% However, as previously discussed, both components put some restrictions on the space from which the strategy is chosen.
%
In particular, OMD with log-barrier requires selecting strategies from the truncated simplex $\Omega$, while a safe decision-making involves choosing strategies from the safe decision space $\mathcal{S}_t$. Unfortunately, these two sets may in general be disjoint.
Our algorithms implement a procedure that enables these two elements to work together without conflict. Specifically, at each round $t \in [T]$, they employ a larger safe decision space $\mathcal{S}_t^\circ$, which is defined as 
$
    %\label{eq:safe_set_larger}
    \mathcal{S}_t^\circ \coloneqq \left\{\xvec\in \Delta_K : { (\widehat{\gvec}_{t,i}-\boldsymbol{\beta}_t)^\top \xvec\leq\alpha_i+\frac{K}{T} \,\,\, \forall i\in[m]}\right\}.
$
This decision space is strictly larger than the safe decision space $\mathcal{S}_t$. In the following lemma, we characterize the decision space obtained by intersecting $\Omega$ and $\mathcal{S}_t^\circ$, which we call the \textit{truncated safe decision space}.
\begin{restatable}{lemma}{solution}
	\label{lem:solution}
    For every $t \in [T]$, let $\widetilde{\mathcal{S}}_t \coloneqq \Omega \cap \mathcal{S}_t^\circ$ be the intersection of the truncated simplex and the safe decision space.
	Then, under $\mathcal{E}(\delta)$, it holds that $\cap_{t\in[T]}\widetilde{\mathcal{S}}_t$ is non-empty.
\end{restatable}
The above lemma can be proven by showing that any $\xvec^\circ \in \Delta_K$ that satisfies the constraints in expectation is included in $\mathcal{S}_t^\circ$. Moreover, $\|\xvec^\circ-\widetilde\xvec\|_1\leq \frac{K}{T}$, where $\widetilde\xvec\coloneqq\argmin_{\xvec\in\Omega}\|\xvec-\xvec^{\circ}\|_1$. Thus, $\widetilde{\xvec}$ belongs to both $\Omega$ and $\mathcal{S}_t^\circ$. Additional details can be found in Appendix \ref{App:soft}. Intuitively, Lemma~\ref{lem:solution} states that, with high probability, at every round $t\in[T]$ there exists a strategy satisfying both the requirements of OMD with log-barrier and a ``suitably-relaxed'' safety condition. %This result plays a crucial role in the construction of algorithms suited the constrained MAB problem that enjoy small-loss style guarantees.
\section{Warm-Up: Small-Loss Guarantees in MABs with Soft Constraints}
\label{sec:soft_main}
% ---------- highlight environments ----------------------
\NewEnviron{highlightblue}{
\noindent
  \tikz[]{
    \node[
      fill=lightblue!30,
      rounded corners,
      inner sep=0pt,
      text width=\linewidth,
      align=left
    ] (X) {\BODY};
  }
}

\NewEnviron{highlightgreen}{
\noindent
  \tikz[]{
    \node[
      fill=green!10,
      rounded corners,
      inner sep=0pt,
      text width=\linewidth,
      align=left
    ] (X) {\BODY};
  }
}

\begin{algorithm}[!htp]
	\caption{\algnamesoft}
	\label{alg: bounded violations}
    \begin{algorithmic}[1]
		\Require Learning rate $\eta > 0$, confidence $\delta \in (0,1)$, thresholds $\{\alpha_i \}_{i \in [m]}$ 
		
		\State Define increase factor $\kappa\gets e^{\frac{1}{\ln T}}$ \label{alg1:line2}
		
		\State Initialize $\xvec_1\gets \frac{1}{K}\boldsymbol{1}$, $h_{1,a}\gets 2K$, $\eta_{1,a}\gets\eta$ for all $a\in[K]$, $\widehat{\gvec}_{1,i}\gets 0$ for all $i\in[m]$, $\boldsymbol\beta_1\gets \boldsymbol{1}$ \label{alg1:line3}
		
		\For{$t	\in[T]$}
		\State Select action $a_t\sim\xvec_t$ \label{alg1:line5}
		 
		\State Observe loss $\ell_t(a_t)$ and constraint costs $\boldsymbol{g}_{t,i}(a_t),\forall i \in [m]$ \label{alg1:line6}
        \Statex \begin{highlightblue}
        \State Update $\widehat{\gvec}_{t,i}$ and $\boldsymbol{\beta}_t$ as described in Section~\ref{sec:cost_estimation} \label{alg1:line8} 

        \State Compute the safe decision space ${\mathcal{S}}_t^\circ$ as described in Section~\ref{sec:truncated_safe}\label{alg1:line85} 

        \State Compute the truncated safe decision space $\widetilde{\mathcal{S}}_t \gets \Omega \cap \mathcal{S}_t^\circ$\label{alg1:line9}
        \end{highlightblue}
	\If{$\widetilde{\mathcal{S}}_t$ is not empty \label{alg1:line10}}%
        \Statex \begin{highlightgreen} 
        \State Compute $\widehat{\ell}_t(a)\gets\frac{\ell_t(a)\mathbbm{1}_{\{a_t=a\}}}{x_t(a)}, \forall a\in[K]$ \label{alg1:line7}
         
        \State $\xvec_{t+1}\gets\argmin_{\xvec\in\widetilde{\mathcal{S}}_t} \widehat{\lvec_t}^\top \xvec + D_{\psi_t}(\xvec, \xvec_t),$ where $\psi_t(\xvec)=\sum_{a=1}^K\frac{1}{\eta_{t,a}}\ln\frac{1}{x(a)}$  \label{alg1:line11}
		
		\For{$a\in[K]$}
		\If{$\frac{1}{x_{t+1}(a)}>h_{t,a}$ \label{alg1:line13}} 
		\State $h_{t+1,a}\gets\frac{2}{x_{t+1}(a)}$, $\eta_{t+1,a}\gets\eta_{t,a} \kappa$ 
		\Else 
		 \State $h_{t+1,a}\gets h_{t,a}$, $\eta_{t+1,a}\gets\eta_{t,a} $ \label{alg1:line17}
         \EndIf
         \EndFor
        \end{highlightgreen}
	\Else 
		 \State Select strategy $\xvec_{t+1}\sim \Omega$ randomly \label{alg1:line22}\;
	\EndIf
    
	\EndFor
    \end{algorithmic}
	\end{algorithm}
As a preliminary step toward our main result for the hard constraints settings, we design an algorithm with small-loss regret bound for \emph{soft constraints} settings.
Although these settings are technically easier than hard constraints ones, our algorithm incorporates some key components that are also central to the algorithm for hard constraints settings designed in Section~\ref{sec:hard_main}. 
Moreover, the algorithm may also be of independent interest for other learning settings with soft constraints.
%
% Thus, we use this simpler scenario to describe the functioning of the core components of our algorithms. Finally, we remark that solving the MAB with soft constraints is of independent interest.

\subsection{The \algnamesoft Algorithm}

Algorithm~\ref{alg: bounded violations} provides the pseudo-code of \emph{Constrained OMD with Log-Barrier} (\algnamesoft). 
At a high level, the algorithm implements the combination of OMD with log-barrier and safe decision spaces introduced in Section~\ref{sec:safe_playground}.
Indeed, Algorithm~\ref{alg: bounded violations} is conceptually split into two main blocks.
One block contains the set of instructions necessary to control constraint violations, highlighted in \textcolor{blue!40!black!50}{blue}.
The other block, highlighted in \textcolor{green!40!black!50}{green}, defines the OMD with log-barrier sub-routine.
%, which performs regret minimization in the truncated safe decision space built in the other block.
% The algorithm builds on existing procedures that achieve a small-loss regret bound in the MAB framework (\emph{e.g.}, Algorithm 1 in ~\citep{lee2020bias}), by incorporating the necessary tools to handle unknown constraints.
% Intuitively, Algorithm~\ref{alg: bounded violations} performs OMD with log-barrier and increasing learning rate on a dynamic decision space. Such a decision space is built to be optimistic with respect to the constraints satisfaction, namely it is a super-set of the true safe decision space with high-probability.
%
Algorithm~\ref{alg: bounded violations} first defines a factor $\kappa$ that is employed to increase the learning rate $\eta_{t,a}$ for action $a$ if the probability of choosing the action falls below a certain threshold $h_{t,a}$ (Line~\ref{alg1:line2}).
%
% The initial strategy is initialized uniformly, and the learning rate is initialized as given in input (Line~\ref{alg1:line3}).
%
At each round, after playing an action and observing some feedback (Lines~\ref{alg1:line5}-\ref{alg1:line6}), the algorithm updates empirical means and confidence bounds for constraint costs (Line~\ref{alg1:line8}). Then, it builds the \emph{truncated safe decision space} $\widetilde{\mathcal{S}_t}$ (Line~\ref{alg1:line9}), as described in Section~\ref{sec:safe_playground}.
If this set is \emph{not} empty (Line~\ref{alg1:line10}), then an update of OMD with log-barrier is performed over $\widetilde{\mathcal{S}}_t$ (Lines~\ref{alg1:line7}-\ref{alg1:line11}).
Moreover, if the probability specified by computed strategy is too \textit{small} for some action $a \in [K]$, \emph{i.e.}, $\nicefrac{1}{x_{t+1}(a)}\geq h_{t,a}$, then the learning rate $\eta_{t,a}$ is increased by a $\kappa$ factor and the threshold is increased to $\nicefrac{2}{x_{t+1}(a)}$ (Line~\ref{alg1:line13}-\ref{alg1:line17}).
Finally, if $\widetilde{\mathcal{S}}_t$ is empty, then a strategy is sampled randomly from the truncated simplex (Line~\ref{alg1:line22}).

%\subsection{Theoretical Guarantees of \algnamesoft}
In the following, we provide the theoretical guarantees attained by \algnamesoft. 
We start by providing an upper bound on the cumulative positive constraint violations attained by the algorithm.
\begin{restatable}{theorem}{violationsoft}
	\label{thm:viol}
	Let $\delta \in (0,1)$. Then, with probability at least $1-2\delta$, the \algnamesoft algorithm suffers cumulative positive constraint violations $V_T \le \mathcal{O}\left( \sqrt{KT\ln\left(\nicefrac{TKm}{\delta}\right)} \right).$
\end{restatable}
The proof of Theorem~\ref{thm:viol} relies on Lemma~\ref{lem:solution}, %Indeed, under the clean event, it holds that the per-round violation is bounded as $	\left[\gvec_{{i}}^\top \xvec_t -\alpha_{{i}}\right]^+ \le \frac{K}{T} + 2\boldsymbol{\beta}_{t-1}^\top \xvec_t$, for all $i\in[m]$. Summing over the the rounds and applying the well-known Azuma inequality, we show that $V_T\leq K + \sum_{t=1}^T\sum_{a=1}^K2{\beta}_{t-1}(a) \mathbbm{1}_{\{a_t=a\}} + \mathcal{O}(\sqrt{T})$, with probability at least $1-2\delta$. 
and it can be found in Appendix \ref{proofs:violation}. The bound provided in the theorem matches, up to constant factors, the lower bound provided in Theorem 3 of \citep{safetree}.
We are now ready to prove the main result of this section, which is a small-loss bound on the cumulative regret suffered by Algorithm~\ref{alg: bounded violations}. This is done by means of the following theorem.
\begin{restatable}{theorem}{regretsoft} 
	\label{thm:reg_soft}
	Let $\delta \in (0,1)$ and $\eta= \min\left\{\nicefrac{1}{40H\ln T\ln\left(\nicefrac{H}{\delta}\right)}, \sqrt{\nicefrac{K}{\sum_{t=1}^{T}\lvec_t^\top\xvec^* \ln\left(\nicefrac{1}{\delta}\right)}}\right\}$, where $H \coloneqq \ln\left(\nicefrac{\lceil\ln(T)\rceil \lceil3\ln(T)\rceil}{\delta}\right)$. Then, with probability at least $1-4\delta$, \algnamesoft suffers a cumulative regret bounded as
	%\label{eq:regret_soft}
	$R_T(\lvec_{1:T}) \le \widetilde{\mathcal{O}}\left( \sqrt{K\sum_{t=1}^{T}\lvec_t^\top\xvec^* \ln\left(\nicefrac{1}{\delta}\right)}\right).
    $
    % when instantiated with $\eta= \min\left\{\nicefrac{1}{40H\ln T\ln\left(\frac{H}{\delta}\right)}, \sqrt{\nicefrac{K}{\sum_{t=1}^{T}\lvec_t^\top\xvec^* \ln\left(\frac{1}{\delta}\right)}}\right\}$.% and
    %where $\widetilde{\mathcal{O}}$ hides logarithmic terms.
\end{restatable}
Theorem~\ref{thm:reg_soft} is proved by noticing that, even though the strategy update of Algorithm~\ref{alg: bounded violations} works on changing decision spaces $\widetilde{\mathcal{S}}_t$, projecting does \emph{not} prevent the OMD sub-routine to guarantee small-loss bounds. The proof can be found in Appendix \ref{proofs:regretsoft}.
Some remarks are in order.
\begin{remark}[Tightness of the bound in Theorem~\ref{thm:reg_soft}]
    %Theorem~\ref{thm:reg_soft} provides a small-loss regret bound for \algnamesoft. These kind of bounds, which depend on the specific problem instance, present a smaller rate than their worst-case counterparts, by construction. Indeed, since the losses are bounded by $1$, the cumulative loss of the best fixed strategy is always smaller than $T$. O
    Let us remark that the bound presented in Theorem \ref{thm:reg_soft} is tight, up to constants and logarithmic terms.
    Indeed, Theorem 3 of \citep{gerchinovitz2016refined} provides a regret lower bound of $\Omega(\sqrt{KL^*})$ in unconstrained MABs, where $L^*$ is the cumulative loss of an action that is optimal in hidsight.
    The bound in Theorem~\ref{thm:reg_soft} scales with the total loss of an optimal-in-hindsight strategy $\xvec^*$, which can be randomized.
    However, since in unconstrained settings an optimal (randomized) strategy is as powerful as  an optimal action, the lower bound in Theorem 3 of \citep{gerchinovitz2016refined} carries over to our setting as well.
\end{remark}
\begin{remark}[Knowledge of $\sum_{t=1}^T \boldsymbol{\ell}_t^\top \xvec^*$]
\label{remark:knowledge}
    Assuming knowledge of $L^*$ to set the learning rate is standard in the literature on small-loss bounds. As discussed in Remark 2 of \citep{allenberg2006hannan} and Remark 1 of \citep{lee2020bias}, a doubling trick can relax this requirement, while ensuring that the regret bound does not deteriorate. This procedure is described in Appendix C.3 of \citep{lee2020closer}. In our case, we require that $\sum_{t=1}^T \boldsymbol{\ell}_t^\top \xvec^*$ is known, instead of $L^*$. However the considerations made for $L^*$ still hold, and our algorithms can be made adaptive w.r.t. this quantity.
\end{remark}

\section{Data-Dependent Guarantees in MABs with Hard Constraints}
\label{sec:hard_main}

This section is entirely devoted to the main contribution of this paper, which is an algorithm that achieves a (tight) \emph{data-dependent} regret bound in constrained MABs with \emph{hard constraints}.
The section begins by introducing the algorithm, which builds on the \algnamesoft algorithm introduced in Section~\ref{sec:soft_main} for soft constraints. 
%
% design an algorithm that achieves a regret bound characterized by two data-dependent terms.
%
% The first term captures the difficulty of satisfying the constraints, while the second one is more akin to classical small-loss bounds in unconstrained settings, as it encodes the performance of an optimal-in-hindsight strategy.
%
%
After proving the guarantees attained by the algorithm, the section ends by proving a lower bound demonstrating that the regret bound of the algorithm is tight. 
%
% We conclude the section by proving a lower bound showing that the two terms in the regret bound of our algorithm are \emph{not} artifacts of our specific approach and analysis, but rather the fundamental components that inherently characterize the complexities of the problem.
%
%
%In this section, we discuss the concept of small-loss style regret bound in the MAB with hard constraints problem. We propose \algnamehard, a regret minimization strategy that guarantees safety in the presence of hard constraints, while suffering a nearly-optimal small-loss style regret bound. Then, we prove a regret lower bound that scales with instance-dependent quantities. We show that, differently from the non-constrained adversarial MAB problem, the total loss of the best decision-maker in hindsight is not enough the describe the instance difficulty. In this setting, the presence of a strictly feasible solution $\xvec^\diamond$ known to the learner (along with its constraints values $\theta_i$, for all $i\in[m]$) plays a key role in the definition of what makes an instance difficult. 

\subsection{The \algnamehard Algorithm}

\NewEnviron{highlightpurple}{
\noindent
  \tikz[]{
    \node[
      fill=purple!10,
      rounded corners,
      inner sep=0pt,
      text width=\linewidth,
      align=left
    ] (X) {\BODY};
  }
}

\begin{comment}
\NewEnviron{highlightpurplereq}{
\noindent
  \tikz[]{
    \node[
      fill=purple!10,
      rounded corners,
      inner sep=0pt,
      text width=\linewidth,
      align=left
    ] (X) {\BODY};
  }
}
\end{comment}

\begin{algorithm}[!htp]
\caption{\algnamehard}
\label{alg: hard}
\begin{algorithmic}[1]
 \Require 
Learning rate $\eta> 0$, confidence parameter $\delta \in (0,1)$, thresholds $\{ \alpha_i \}_{i \in [m]}$, \emph{strictly} feasible strategy $\xvec^\diamond \in \Delta_K$ with its associated costs $\{\theta_i \}_{i \in [m]}$ (see Assumption~\ref{ass:knowldege}) \label{line:require_hard} 
	
	\State Define increase factor $\kappa\gets e^{\frac{1}{\ln T}}$ 
	
	\State Initialize $\widetilde\xvec_1\gets\frac{1}{K}\boldsymbol{1}$, $\rho_{1,a}\gets2K$, $\eta_{1,a}\gets\eta$  for all $a\in[K]$, $\widehat{\gvec}_{1,i}\gets0$ for all $i\in[m]$, $\boldsymbol\beta_1\gets\boldsymbol{1}$
    \Statex \begin{highlightpurple}
    \State Initialize $\gamma_0\gets\max_{i\in[m]}\frac{1-\alpha_i}{1-\theta_i}$ \label{line:init_hard_1}
	
	\State Select $\xvec_1\gets\gamma_0\xvec^\diamond + (1-\gamma_0)\widetilde\xvec_1$ \label{line:init_hard_2}
    \end{highlightpurple}
	\For{$t	\in[T]$} 
		\State Select action $a_t\sim\xvec_t$ \label{alg2:line5}
		 
		\State Observe loss $\ell_t(a_t)$ and constraint costs $\boldsymbol{g}_{t,i}(a_t),$ $\forall i \in [m]$ 
        \Statex \begin{highlightblue}
        \State Update $\widehat{\gvec}_{t,i}$ and $\boldsymbol{\beta}_t$ as described in Section~\ref{sec:cost_estimation}\label{alg2:line8} 

        \State Compute the safe decision space ${\mathcal{S}}_t^\circ$ as described in Section~\ref{sec:truncated_safe} \label{alg2:line85} 

        \State Compute the truncated safe decision space $\widetilde{\mathcal{S}}_t \gets \Omega \cap \mathcal{S}_t^\circ$\label{alg2:line9}
        \end{highlightblue}
\If {$\widetilde{\mathcal{S}}_t$ is not empty} \label{alg2:line10}
    \Statex \begin{highlightgreen}
        \State   Compute $\widehat{\ell}_t(a)\gets\frac{\ell_t(a)\mathbbm{1}_{\{a_t=a\}}}{x_t(a)},\forall a\in[K]$ \label{alg2:line7}
         
        \State $\xvec_{t+1}\gets\argmin_{\xvec\in\widetilde{\mathcal{S}}_t} \widehat{\lvec_t}^\top \xvec + D_{\psi_t}(\xvec, \xvec_t)$, where  $\psi_t(\xvec)=\sum_{a=1}^K\frac{1}{\eta_{t,a}}\ln\frac{1}{x(a)}$  \label{alg2:line11}
		
    \For{$a\in[K]$}
	\If{$\frac{1}{x_{t+1}(a)}>h_{t,a}$} \label{alg2:line13}
	\State $h_{t+1,a}\gets\frac{2}{x_{t+1}(a)}$,          $\eta_{t+1,a}\gets\eta_{t,a} \kappa$   
	\Else 
	\State $h_{t+1,a}\gets h_{t,a}$, $\eta_{t+1,a}\gets\eta_{t,a} $\label{alg2:line17}
	\EndIf
    \EndFor
    \end{highlightgreen}
    \Statex \begin{highlightpurple}
    \State Compute the combination factor:
	$$
        \gamma_t \gets \begin{cases}
		\max_{  i\in[m]}\left\{\frac{\min\{(\widehat{\gvec}_{t,i}+\boldsymbol{\beta}_t)^\top \widetilde\xvec_{t+1},1\}-\alpha_i}{\min\{(\widehat{\gvec}_{t,i}+\boldsymbol{\beta}_t)^\top \widetilde\xvec_{t+1},1\}-\theta_i}\right\} &   \text{if} \ \ \mathcal{E} \text{ holds}\\
				0 & \text{otherwise}
	\end{cases},\;$$ \label{line:combination_factor} 
        \State where $\mathcal{E}= \{\exists i\in[m]: (\widehat{\gvec}_{t,i}+\boldsymbol{\beta}_t)^\top \widetilde\xvec_{t+1} >\alpha_i\}$
    \Else \end{highlightpurple}
    \State Select strategy $\widetilde{\xvec}_{t+1}\sim \Omega$ randomly
    \Statex \begin{highlightpurple}
    \State Set the combination factor to $\gamma_t\gets1$ \label{line:gamma_1}
    %\end{highlightpurple}
\EndIf
%\Statex \begin{highlightpurple}
\State
$\xvec_{t+1}\gets \gamma_t\xvec^{\diamond}+(1-\gamma_t)\widetilde{\xvec}_{t+1}$ \label{line:combination} 
\EndFor \end{highlightpurple}
\end{algorithmic}
\end{algorithm}

Algorithm~\ref{alg: hard} provides the pseudo-code of \emph{Safe OMD with log-barrier} (\algnamehard).
Notice that the algorithm takes additional inputs compared to \algnamesoft.
Specifically, it takes as inputs a \emph{strictly feasible} strategy, \emph{i.e.}, a strategy $\xvec^\diamond \in \Delta_K$ as defined in Assumption~\ref{ass:knowldege}, and its associates costs $\{ \theta_i\}_{i \in [m]}$.

Algorithm \ref{alg: hard} highlights in \textcolor{purple!30}{pink} its differences with respect to Algorithm \ref{alg: bounded violations}.
The key difference between \algnamesoft and \algnamehard is that the strategy chosen by the latter is \emph{not} readily the one selected through the OMD update.
Specifically, at each round $t $, \algnamehard plays a convex combination between the strictly feasible strategy $\xvec^\diamond$ given as input and the one selected by OMD, denoted $\widetilde\xvec_t$.
The combination factor $\gamma_{t}$ is chosen in an adaptive way to guarantee that the resulting strategy $\xvec_{t}$ satisfies the constraints with high probability.
Intuitively, the combination factor $\gamma_{t}$ weights how safe is to play the strategy computed by OMD rather than the strictly feasible strategy.
If the strategy $\widetilde{\xvec}_{t}$ produced by the OMD update satisfies the constraints with high probability, then $\gamma_{t}$ is set to zero, and the algorithm selects strategy $\widetilde\xvec_t$. 
A larger $\gamma_{t}$ weights more the strictly feasible strategy $\xvec^\diamond$ than $\widetilde\xvec_t$. 
In the first rounds, since confidence intervals for cost estimates are large, $\gamma_{t}$ is strictly greater than zero, while, as the confidence intervals become smaller, $\gamma_t$ approaches zero.
As in Algorithm \ref{alg: bounded violations}, when the truncated safe decision space is \emph{not} empty, the algorithm runs a step of the OMD with log-barrier sub-routine.
Then, a new combination factor is computed (Line \ref{line:combination_factor}).
%
%Intuitively, the combination factor $\gamma_t$ weights how safe is the strategy computed by OMD with respect to the strictly feasible strategy. If the strategy $\widetilde{\xvec}_{t+1}$ produced by OMD is considered strictly safe with high probability then $\gamma_t$ is set to $1$ and, according to Line \ref{line:combination}, the played strategy is actually this one.
%
On the other hand, if the truncated safe decision space is empty, then $\gamma_t$ is set to zero (Line \ref{line:gamma_1}), and the algorithm thus selects the known strictly feasible strategy.
In Figure~\ref{fig:1} in Appendix \ref{app:image}, we provide a graphical intuition on how $\xvec_t$ is selected when $K=3$.

\subsection{Theoretical Guarantees of \algnamehard}
In this section, we provide the theoretical guarantees attained by \algnamehard. 
%\paragraph{Safety Guarantee} 
We start by showing that \algnamehard satisfies the constraints at every round with high probability.
\begin{restatable}{theorem}{safety}
\label{thm:safety}
	Let $\delta\in(0,1)$. With probability at least $1-\delta$, \algnamehard guarantees that $ \gvec_i^\top\xvec_{t}\leq \alpha_i$ holds for every action $a \in [K]$, constraint $i \in [m]$, and round $t \in [T]$.
\end{restatable}
Theorem \ref{thm:safety} can be proven by analyzing the behavior of the combination factor $\gamma_t$. We show that $\gamma_t$ is large enough to compensate the violations potentially suffered by the strategy $\widetilde\xvec_{t+1}$ computed by the OMD update. Specifically, $\gamma_t=0$ when $\widetilde{\xvec}_{t+1}$ satisfies the constraints with high probability. Otherwise, $\gamma_t$ is proportional to the pessimistic violation that $\widetilde{\xvec}_{t+1}$ would suffer. Assuming that $\xvec^\diamond$ is strictly feasible implies that $\gamma_t<1$ in every round $t$ in which the truncated safe decision space is non-empty. To see this, notice that $\gamma_t\leq \max_{i\in[m]}\frac{1-\alpha_i}{1-\theta_i}\leq\max_{i\in[m]} \frac{1-\alpha_i}{1-\alpha_i+\rho}<1$. 
Thus, a minimum amount of exploration is always guaranteed. When the truncated safe decision space is empty, the algorithm uses the strictly feasible strategy. %Thus that portion of rounds is guaranteed to be safe.

%\paragraph{Regret Analysis} 
We now analyze the regret suffered by \algnamehard. Before presenting our main result, we introduce some technical lemmas that are useful to understand the nature of the regret bound.
\begin{restatable}{lemma}{firstpart}
	\label{regret:first_part}
	Let $\delta\in(0,1)$ and $\rho\geq \frac{12K}{T}$. Then, with probability at least $1-2\delta$, \algnamehard satisfies:
	\begin{equation}
    \label{eq:gamma_t}
		R_T^\diamond(\lvec_{1:T})\leq {\mathcal{O}}\left(\frac{K}{\rho}\sqrt{\sum_{t=1}^T\left(\lvec_t^\top(\xvec^\diamond - \xvec^*)\right)^2}\cdot\ln\left(\frac{KTm}{\delta}\right) + \frac{K}{\rho^6}\ln\left(\frac{KTm}{\delta}\right)\right),
	\end{equation}
    where $R_T^\diamond(\lvec_{1:T}) \coloneqq \sum_{t=1}^T \gamma_{t-1}{\boldsymbol\ell}_t^\top(\xvec^\diamond-\xvec^*)$.%, and ${\mathcal{O}}$ hides universal constants.
\end{restatable}
% \noindent \emph{Proof Idea} This result can be obtained by analyzing separately the rounds in which $\gamma_{t-1}\le \frac{1}{2}$ and remaining ones. In the former set of rounds, we can apply the Cauchy-Schwarz inequality to the LHS of Equation \eqref{eq:gamma_t} and bound $\sum_{t: \gamma_{t-1}\le 1/2} \gamma_{t-1}^2 \le \mathcal{O}(K\ln(KTm/\delta)/\rho)$. It can be shown that the number of rounds in which $\gamma_{t-1}> \frac{1}{2}$ is at most logarithmic in $T$, thus the contribution to the RHS of Equation \eqref{eq:gamma_t} is dominated by the remaining rounds. Details on the proof of Lemma \ref{regret:first_part} can be found in Appendix \ref{App:hard}.
%
Intuitively, the above lemma states that the regret accrued by the strictly feasible strategy, when weighted by the combination factors, is not \emph{too large}. Specifically, the term that appears in Equation~\eqref{eq:gamma_t} is the scaled Euclidean norm of the sequence of instantaneous regrets suffered by $\xvec^\diamond$. This quantity represents some sort of distance between an optimal strategy and the strictly feasible one. At the end of this section, we provide a discussion on the role of this quantity, showing that it represents a source of complexity for the problem instance.

\begin{restatable}{lemma}{secondpart}
	\label{regret:second_part}
	Let $\delta\in(0,1)$ and $\eta \le \frac{\rho}{40H\ln T\ln\left(\nicefrac{H}{\delta}\right)}$, where $H \coloneqq \ln\left(\nicefrac{\lceil\ln(T)\rceil \lceil3\ln(T)\rceil}{\delta}\right)$. Then, with probability at least $1-2\delta$, \algnamehard satisfies:
	\begin{align*}
		\widetilde R_T(\lvec_{1:T})\leq{\mathcal{O}}\Bigg(\frac{K}{\eta}+ \frac{\eta}{\rho}\sum_{t=1}^T\lvec_t^\top\xvec_t+ \frac{K\eta}{\rho}\ln\left(\frac{1}{\delta}\right) +\sqrt{\sum_{t=1}^{T}\lvec_t^\top\xvec_t\ln\left(\frac{1}{\delta}\right)}+\frac{\eta}{\rho}\sum_{t=1}^T\lvec_t^\top\xvec^*\ln\left(\frac{1}{\delta}\right)\Bigg),
	\end{align*}
	where $\widetilde R_T(\lvec_{1:T})=\sum_{t=1}^T (1-\gamma_{t-1}) \lvec_t^\top(\widetilde{\xvec}_t - \xvec^*)$.
\end{restatable}
Lemma \ref{regret:second_part} provides an upper bound on the regret accrued by the strategy proposed by OMD, when weighted by $1-\gamma_t$. The bound seems involved. However, the terms we are more interested in are the second and the fourth, which dominate the other ones. In particular, we highlight how the quantity $\widetilde R_T(\lvec_{1:T})$ is bounded by a sum of the total loss of the strategy played by \algnamehard and the total loss of the optimal strategy. This provides a link between the total loss of the strategy proposed by OMD, the total loss of the strategy that is actually played by \algnamehard, and the total loss of the optimal strategy. Keeping in mind that our goal is to obtain a regret bound for \algnamehard that scales with the latter, Lemma \ref{regret:second_part} plays a key role in the final result. The proofs of the lemmas are in Appendix~\ref{proofs:regrethard}.

We are now ready to present our main result, \emph{i.e.}, a high-probability regret bound for \algnamehard.
\begin{restatable}{theorem}{regrethard}
	\label{thr:ub_hard}
	Let $\delta\in(0,1)$, $\rho\geq \frac{12K}{T}$, and $\eta= \min\left\{\frac{\rho}{40H\ln T\ln\left(\nicefrac{H}{\delta}\right)}, \sqrt{\nicefrac{K}{\sum_{t=1}^{T}\lvec_t^\top\xvec^* \ln\left(\nicefrac{1}{\delta}\right)}}\right\}$, where $H \coloneqq \ln\left(\nicefrac{\lceil\ln(T)\rceil \lceil3\ln(T)\rceil}{\delta}\right)$. Then, \algnamehard
    suffers a cumulative regret bounded as:
	\begin{equation}
		\label{eq:regrethard}
		R_T(\lvec_{1:T})\leq  {\widetilde{\mathcal{O}}}\Bigg(\underbrace{ \frac{K\ln \left(\nicefrac{1}{\delta}\right)}{\rho}\sqrt{\sum_{t=1}^T\left(\lvec_t^\top(\xvec^\diamond - \xvec^*)\right)^2}}_{\text{\textcolor{red!70!black}{(A) Safety Complexity}}}+\underbrace{\frac{1}{\rho}\sqrt{K\sum_{t=1}^{T}\lvec_t^\top\xvec^* \ln\left(\frac{1}{\delta}\right)}}_{\text{\textcolor{green!50!black}{(B) Bandit Complexity}}}\Bigg),
	\end{equation}
    where $\widetilde{\mathcal{O}}$ hides universal constants and logarithmic terms not depending on $\delta$.
\end{restatable}
Theorem~\ref{thr:ub_hard} is proved by decomposing the regret in the quantities analyzed in Lemmas~\ref{regret:first_part}~and~\ref{regret:second_part}. By construction, $\xvec_t = \gamma_{t-1}\xvec^\diamond + (1-\gamma_{t-1})\widetilde{\xvec}_t$, which implies $R_T(\lvec_{1:T}) = R_T^\diamond(\lvec_{1:T})+\widetilde{R}_T(\lvec_{1:T})$. Thus, we can sum the upper bounds presented in Lemmas~\ref{regret:first_part}~and~\ref{regret:second_part}. Finally, noting that, by definition of $\eta$, it holds $ \frac{\eta}{\rho}\sum_{t=1}^T\lvec_t^\top\xvec_t \le \frac{1}{2}R_T + \frac{\eta}{\rho}\sum_{t=1}^T\lvec_t^\top\xvec^*$, we can solve the deriving quadratic inequality in $R_T$, which yields Equation \eqref{eq:regrethard}. A detailed proof of Theorem \ref{thr:ub_hard} is in Appendix \ref{proofs:regrethard}.
Theorem~\ref{thr:ub_hard} is one of the main results of this paper. It establishes a regret bound that depends on two contributions: the total loss incurred by an optimal-in-hindisght strategy \textcolor{green!50!black}{(B)}, and the total squared difference between the losses of the strictly feasible strategy and the benchmark \textcolor{red!70!black}{(A)}. This result provides a natural interpretation on the intrinsic difficulty of MABs with hard constraints. On the one hand, our bound scales as the performance of an optimal strategy, which is common to any small-loss bound in unconstrained MABs. We call this contribution \textcolor{green!50!black}{Bandit Complexity}, as it represents the complexity of learning independently of the presence of the constraints. On the other hand, we pay an additional term---peculiar of our setting---that encodes 
the distance between the benchmark and the strictly feasible strategy given as input to the algorithm. We call this contribution \textcolor{red!70!black}{Safety Complexity}, as it represents the complexity of learning an optimal feasible strategy while satisfying the constraints at every round with high probability. %We conclude this section by remarking that \algnamehard, as already discussed for \algnamesoft in Remark \label{remark:knowledge}, can be adapted to be unaware of the quantity $\sum_{t=1}^T \boldsymbol{\ell}_t^\top\xvec^*$, which is used to define the learning rate.
We acknowledge that the state-of-the-art bounds in online settings with hard constraints are of order $\widetilde{\mathcal{O}}(\nicefrac{1}{\rho}\sqrt{T})$ (\emph{e.g.},~\citet{pacchiano}). Indeed, our bound not only improves the aforementioned result in the best case, while being equivalent in the worst one, but it also decomposes the former in two quantities that are easily interpretable, and possibly very different one from the other.

\subsection{A Small-Loss Style Regret Lower Bound}
In this section, we provide a small-loss style regret lower bound for the hard constrained bandit problem. In~\citep{gerchinovitz2016refined} the authors show that, in adversarial non-constrained bandit problems, the regret suffered by every algorithm is lower bounded as $\mathcal{O}(\sqrt{L^*})$, where $L^*$ is the total loss accrued by the benchmark. This term is represented, in our setting, by the \textcolor{green!50!black}{Bandit Complexity} contribution. In our setting, the strictly feasible strategy $\xvec^\diamond$ is crucial in defining how hard is an instance: in fact, Equation \eqref{eq:regrethard} shows that \algnamehard benefits instances where the optimum is \textit{close} (in terms of performances) to the strictly feasible strategy, which translates in a small \textcolor{red!70!black}{Safety Complexity}. This behavior is natural as the strictly feasible strategy represents a starting point for the exploration, and the algorithm remains somehow tied to that. This raises the natural question on whether this double dependency, one on the optimal total loss and the other on the difference with the strictly feasible strategy, is actually tight. The next two results bridge Theorem \ref{thr:ub_hard} with the standard literature results, \emph{i.e.} regret bounds depending on $T$, and show that the performance of \algnamehard is optimal when disregarding logarithmic terms. We start by introducing an important technical notion, that is handy in bridging constrained small-loss bounds and standard bounds depending on $T$ only.
We define the \textit{constrained small-loss balls} as $
	\mathcal{B}_{\omega, \Delta, T} \coloneqq \left\{ \lvec_{1:T} \in [0,1]^{KT} : \frac{\sum_{t=1}^T \lvec_t^\top \xvec^*}{T} \le \omega \cap \frac{\sum_{t=1}^T (\lvec_t^\top (\xvec^\diamond-\xvec^*))^2}{T} \le \Delta^2 \right\}.
$
The quantities $\omega, \Delta \in [0,1]$ represent two different sources of difficulty for the learning algorithm: $\omega$ express how difficult is the identification of the optimum as in a non-constrained bandit problem, and $\Delta$ represents the additional difficulty provided by the constraints satisfaction. The next result allows us to rephrase the regret upper bound provided in Theorem \ref{thr:ub_hard} in terms of the constrained small-loss ball. In fact, it is a trivial consequence of Equation \eqref{eq:regrethard} and the definition of constrained small-loss ball.
\begin{corollary}
	For all $ \omega \in [0,1]$ and for all $ \Delta \in [0,1]$, it holds $\sup_{\lvec_{1:T} \in \mathcal{B}_{\omega,\Delta, T}} \mathbb{E}[R_T(\lvec_{1:T})] \le \widetilde{\mathcal{O}}\left(\frac{K\Delta}{\rho}\sqrt{T}+ \frac{1}{\rho} \sqrt{K\omega T}\right),
	$
    where the expectation is taken with respect to the internal randomization of the algorithm.
\end{corollary}
Notice that $\omega$ and $\Delta$ can be treated as instance-dependent parameters, as they represent how far an instance is from the worst-case one, \emph{i.e.} when those two are both equal to $1$.
Finally, the next result shows that no algorithm can have a better dependence on the parameters $\omega$ and $\Delta$.
\begin{restatable}{theorem}{lowerbound}
	\label{thr:lb}
	Let $K\ge 2$, $T \ge \max\left\{2, (11+\ln T)\left(\frac{8}{3}\right)^2\right\}$, and $\omega \in \left[\frac{1}{T}\left(\frac{11}{2}+\ln T\right), \frac{1}{2}\right]$. Then for every randomized algorithm, we have
	$
		\sup_{\lvec_{1:T} \in \mathcal{B}_{\omega,\Delta, T}} \mathbb{E}[R_T(\lvec_{1:T})] \ge \Omega\left(\frac{\Delta}{\rho}\sqrt{T}+  \sqrt{\omega T}\right),
	$
	where the expectation is taken with respect to the internal randomization of the algorithm.
\end{restatable}
A detailed proof of Theorem \ref{thr:lb} can be found in Appendix \ref{proofs:lowerbound}. This result shows that \algnamehard achieves an optimal dependence on both $\Delta$ and $\omega$. 

Finally, we leave as an open question whether the dependence on the constant $\rho$ in contribution \textcolor{green!50!black}{(B)} of the regret upper bound can be removed. Notice that $\rho$ is a constant tied to the strictly feasible strategy and it encompasses the difficulty provided by the constraints to the instance. Thus, we believe that the $\nicefrac{1}{\rho}$ dependence should affect contribution \textcolor{red!70!black}{(A)} only and we conjecture that our lower bound is actually tight, while it is the upper bound that can be lowered.

\bibliographystyle{plainnat}
\bibliography{bibliography}

\newpage
%%%%%%%%%%%%%%%%%%%%%%%%%%%%%%%%%%%%%%%%%%%%%%%%%%%%%%%%%%%%

\appendix

\section*{Appendix}
The Appendix is structured as follows:
\begin{itemize}
	%\item In Appendix~\ref{App:related}, we provide additional related works.
	\item In Appendix~\ref{app:image}, we provide a graphical representation of Algorithm \ref{alg: hard}'s update.
	\item In Appendix~\ref{app:related}, we provide the complete discussion on related works.
	\item In Appendix~\ref{App:soft}, we provide the omitted analysis for the \emph{soft constraints} setting.
	\item In Appendix~\ref{App:hard}, we provide the omitted analysis for the \emph{hard constraints} setting.
\end{itemize}

\section{Graphical Representation of the Update of Algorithm \ref{alg: hard}}

\label{app:image}
\begin{figure}[!h]
  \centering
    \subfigure[]{
        \label{fig: 1a}
        \begin{tikzpicture}
% Assi cartesiani
	%\draw[->] (-1, 0) -- (5, 0) node[below] {$x$};
	%\draw[->] (0, -1) -- (0, 4) node[left] {$y$};
	
	% Vertici del simplesso
	\coordinate (A) at (0, 0);
	\coordinate (B) at (4, 0);
	\coordinate (C) at (2, 3);
	
	% Simplesso bianco con bordo nero
	\fill[white] (A) -- (B) -- (C) -- cycle;
	\draw[thick] (A) -- (B) -- (C) -- cycle;
	
	% Prima regione colorata (azzurra, tagliata dal primo segmento)
	\fill[blue!30, opacity=0.9]  (B) -- ($(C)!0.7!(B)$) -- (A) -- cycle;
	\draw[thick, blue] (A) -- ($(B)!0.3!(C)$); % Primo segmento (azzurro)
	
	% Seconda regione colorata (verde, quasi parallela al primo segmento)
	%\fill[red!30, opacity=0.5] ($(B)!0!(C)$) -- ($(C)!0.3!(B)$) -- (A) -- cycle;
	%\draw[thick, red] (A) -- ($(B)!0.7!(C)$); % Secondo segmento (verde)
	
	% Punti ed etichette
	\fill[black] ($(B)!0.7!(C)$) circle (2pt) node[above right] {$\widetilde{\boldsymbol{x}}_{t+1}$}; % Punto interno (x^*)
	\fill[black] (B) circle (2pt) node[above right] {$\boldsymbol{x}^\diamond$}; % Punto su segmento blu (x^diamond)
	\fill[black] ($(B)!0.2!(C)$) circle (2pt) node[above right] {${\boldsymbol{x}}_{t+1}$}; 
	
	% Punto vicino a x^diamond e linea tratteggiata verso x^*
	%\coordinate (xt) at ($(B)!0.55!(C)$);
	%\fill[black] (xt) circle (2pt) node[above right] {$\boldsymbol{x}_t$};
	%\draw[dashed] (xt) -- (1.8, 1.1);
\end{tikzpicture}
	
    }
    \subfigure[]{
        \label{fig: 1b}
        \begin{tikzpicture}
	% Assi cartesiani
	%\draw[->] (-1, 0) -- (5, 0) node[below] {$x$};
	%\draw[->] (0, -1) -- (0, 4) node[left] {$y$};
	
	% Vertici del simplesso
	\coordinate (A) at (0, 0);
	\coordinate (B) at (4, 0);
	\coordinate (C) at (2, 3);
	
	% Simplesso bianco con bordo nero
	\fill[white] (A) -- (B) -- (C) -- cycle;
	\draw[thick] (A) -- (B) -- (C) -- cycle;
	
	% Prima regione colorata (azzurra, tagliata dal primo segmento)
	%\fill[blue!30, opacity=0.9]  (B) -- ($(C)!0.7!(B)$) -- (A) -- cycle;
	%\draw[thick, blue] (A) -- ($(B)!0.3!(C)$); % Primo segmento (azzurro)
	
	% Seconda regione colorata (verde, quasi parallela al primo segmento)
	\fill[red!30, opacity=0.5] ($(B)!0!(C)$) -- ($(C)!0.3!(B)$) -- (A) -- cycle;
	\draw[thick, red] (A) -- ($(B)!0.7!(C)$); % Secondo segmento (verde)
	
	% Punti ed etichette
	\fill[black] ($(B)!0.7!(C)$) circle (2pt) node[above right] {$\widetilde{\boldsymbol{x}}_{t+1}$}; % Punto interno (x^*)
	\fill[black] (B) circle (2pt) node[above right] {$\boldsymbol{x}^\diamond$}; % Punto su segmento blu (x^diamond)
	\fill[black] ($(B)!0.2!(C)$) circle (2pt) node[above right] {${\boldsymbol{x}}_{t+1}$}; 
	
	% Punto vicino a x^diamond e linea tratteggiata verso x^*
	%\coordinate (xt) at ($(B)!0.55!(C)$);
	%\fill[black] (xt) circle (2pt) node[above right] {$\boldsymbol{x}_t$};
	%\draw[dashed] (xt) -- (1.8, 1.1);
\end{tikzpicture}
    }
    \subfigure[]{
        \label{fig: 1c}
        \begin{tikzpicture}
	% Assi cartesiani
	%\draw[->] (-1, 0) -- (5, 0) node[below] {$x$};
	%\draw[->] (0, -1) -- (0, 4) node[left] {$y$};
	
	% Vertici del simplesso
	\coordinate (A) at (0, 0);
	\coordinate (B) at (4, 0);
	\coordinate (C) at (2, 3);
	
	% Simplesso bianco con bordo nero
	\fill[white] (A) -- (B) -- (C) -- cycle;
	\draw[thick] (A) -- (B) -- (C) -- cycle;
	
	% Prima regione colorata (azzurra, tagliata dal primo segmento)
	\fill[blue!30, opacity=0.9]  (B) -- ($(C)!0.7!(B)$) -- (A) -- cycle;
	\draw[thick, blue] (A) -- ($(B)!0.3!(C)$); % Primo segmento (azzurro)
	
	% Seconda regione colorata (verde, quasi parallela al primo segmento)
	\fill[red!30, opacity=0.5] ($(B)!0!(C)$) -- ($(C)!0.3!(B)$) -- (A) -- cycle;
	\draw[thick, red] (A) -- ($(B)!0.7!(C)$); % Secondo segmento (verde)
	
	% Punti ed etichette
	\fill[black] ($(B)!0.7!(C)$) circle (2pt) node[above right] {$\widetilde{\boldsymbol{x}}_{t+1}$}; % Punto interno (x^*)
	\fill[black] (B) circle (2pt) node[above right] {$\boldsymbol{x}^\diamond$}; % Punto su segmento blu (x^diamond)
	\fill[black] ($(B)!0.2!(C)$) circle (2pt) node[above right] {${\boldsymbol{x}}_{t+1}$}; 
	
	% Punto vicino a x^diamond e linea tratteggiata verso x^*
	%\coordinate (xt) at ($(B)!0.55!(C)$);
	%\fill[black] (xt) circle (2pt) node[above right] {$\boldsymbol{x}_t$};
	%\draw[dashed] (xt) -- (1.8, 1.1);
\end{tikzpicture}
    }%
    \caption{A graphical representation of the update performed by Algorithm~\ref{alg: hard}. For the sake of exposition, we omit the constraint $x(a)\geq 1/T, \forall a\in[K]$. Specifically, in Figure~\ref{fig: 1a}, we provide the graphical representation of the safe subset of the simplex. Notice that, the strictly safe strategy $\xvec^\diamond$ associated to the Slater's parameter $\rho$ is a vertex of the simplex. In Figure~\ref{fig: 1b}, we provide the graphical representation of the \emph{estimated} safe decision space. Notice that, $\widetilde{\xvec}_{t+1}$ lies on this set as prescribed by the projection of both Algorithm~\ref{alg: bounded violations} and Algorithm~\ref{alg: hard}. Finally, in Figure~\ref{fig: 1c}, we provide how the convex combination is performed. Notice that, $\xvec_{t+1}$ is computed pessimistically. Indeed, due to the high uncertainty in the constraints estimation, $\xvec_{t+1}$ is an interior point of the safe space.}
    \label{fig:1}
\end{figure}
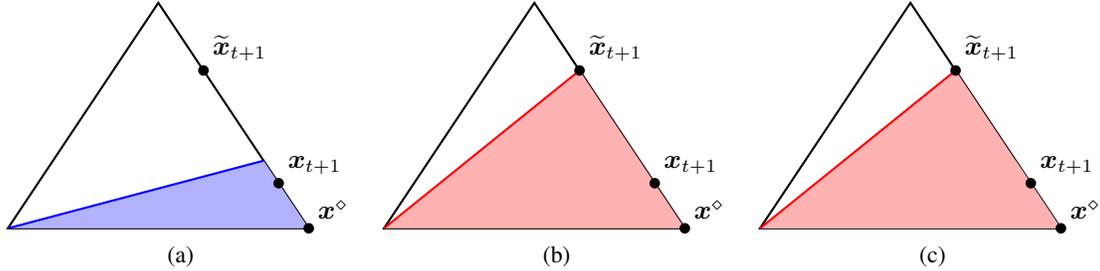
\section{Related Works}
\label{app:related}
%In this section, we provide a brief summary of the existing literature related to this work. %Specifically, we focus on two specific research areas: online learning algorithms with data-dependent regret guarantees and online learning with unknown constraints.
\paragraph{Data-Dependent Regret Bounds}
Over the last two decades, in the literature on (unconstrained) adversarial MABs there has been an increasing interest in providing regret guarantees that depend on the difficulty of the specific instance faced by algorithm. Such bounds are usually referred to as \textit{data-dependent bounds} or \textit{instance-dependent bounds}. Examples include, but are not limited to, \textit{small-loss bounds} (also called \textit{first-order bounds}). An algorithm enjoying small-loss guarantees has its regret upper bound scaling with $\mathcal{O}(\sqrt{L^*})$, where $L^* \le T$ is the total loss accrued by the benchmark. Small loss bounds over the \textit{expected} regret have been obtained in several settings, from MABs to contextual bandits \citep{allenberg2006hannan, neu2015first, allen2018make, lee2020closer}. However, few works managed to recover small-loss guarantees that hold with high probability \citep{neu2015explore, lee2020bias}. %Other relevant types of data-dependent bound are \textit{second order bounds}, \emph{i.e.}, bounds scaling with $\mathcal{O}(\sqrt{Q^*})$, where $Q^* \le L^*$ is the sum of the squared losses accrued by the benchmark \citep{cesa2007improved}. This quantity represents the variability of of the benchmark and characterize the \textit{stationarity} of the instance. Finally, some works provide \textit{zero order bounds}, \emph{i.e.}, bounds scaling with $\mathcal{O}(\sqrt{TM})$, where $M$ is the maximum observed loss in the instance \citep{cesa2007improved}. However, zero and second order bounds have only been obtained in the full-information feedback. 
The tightness of small-loss bounds has been analyzed only recently. In particular, in \citep{gerchinovitz2016refined} the authors provide data-dependent lower bounds for adversarial MABs, showing that the $\mathcal{O}(\sqrt{L^*})$ rate cannot be improved.% Moreover, they also show that zero and second order bounds cannot be obtained in the bandit setting.
\paragraph{Constrained Online Learning}	
Online leaning with \emph{unknown} constraints has been widely investigated (see, \emph{e.g.},~\citep{Mannor, liakopoulos2019cautious, castiglioni2022online, Unifying_Framework}). Two main settings are usually studied. In \emph{soft constraints} settings (see, \emph{e.g.},~\citep{chen2022strategies}), the aim is to guarantee that the constraint violations incurred by the algorithm grow sub-linearly. In \emph{hard constraints} settings, the algorithms must satisfy the constraints at every round, by assuming knowledge of a strictly feasible decision (see, \emph{e.g.},~\citep{pacchiano}). Both soft and hard constraints have been generalized to settings that are more challenging than MABs, such as linear bandits~(see, \emph{e.g.},~\citep{polytopes}) and games (see, \emph{e.g.},~\citep{safetree}). %and constrained MDPs (\emph{e.g.},~\citep{stradi2024learning}) --we refer to Appendix~\ref{App:related} for an extended discussion on constrained Markov decision processes--.
None of these works in the constrained online learning literature studies data-dependent regret bounds. 
There also exists a huge literature on constrained Markov decision processes (see, \emph{e.g.}, \citep{Online_Learning_in_Weakly_Coupled, Constrained_Upper_Confidence, bai2020provably, Exploration_Exploitation, Upper_Confidence_Primal_Dual, ding2021provably, bounded, pmlr-v151-wei22a, Non_stationary_CMDPs,ding_non_stationary, stradi2024learningconstrainedmarkovdecision, muller2024truly,stradi2024learning, stradi24a, stradi2024best,stradi2024optimalstrongregretviolation}). Most of these works focus on stochastic settings, and none of them provides data-dependent regret bounds.
%
%. When dealing with adversarial settings, the metric employed is the cumulative violation which allows for cancellations, thus being weaker than the \emph{soft constraints} setting presented in this work. ~\citet{stradi2024learning} study a setting similar to ours.
%
Finally, some works focus on constrained online convex optimization settings (see, \emph{e.g.},~\citep{mahdavi2012trading, jenatton2016adaptive, yu2017online}). Nonetheless, they do not provide data-dependent regret bounds.

%We refer to Appendix~\ref{App:related} for an extensive discussion on CMDPs.
\section{Omitted Proofs for Soft Constraints}
\label{App:soft}

\subsection{Non-Emptiness of the Truncated Safe Decision Space}
\solution*
\begin{proof}
	Consider a feasible strategy $\xvec^{\circ}$, namely, $\gvec_{i}^\top\xvec^{\circ}\leq \alpha_i$ for all $i\in[m]$. Under the event $\mathcal{E}(\delta)$, Lemma \ref{lemma:global_hoeffding} implies
	\begin{equation}\label{eq1: solution}
		\left(\widehat\gvec_{t,i}-\boldsymbol{\beta}_t\right)^\top\xvec^{\circ}\leq \alpha_i, \quad \forall t\in[T].
	\end{equation}
	Notice that, since $\xvec^*$ is feasible, the aforementioned reasoning still holds for the constrained optimal solution.
	
	Then, notice that for any $\xvec\in\Delta_K$, there exists a strategy $\widetilde{\xvec} \in \Omega$ s.t. $\| \widetilde{\xvec}- \xvec\|_1\le \frac{K}{T}$. 
	
	Thus, employing the reasoning above and taking: \[\widetilde\xvec\coloneq\argmin_{\xvec\in \Omega}\|\xvec-\xvec^{\circ}\|_1,\] we have, for all $i \in [m]$ and for all $t\in[T]$:
	\begin{align*}
		(\widehat{\gvec}_{t,i} - \boldsymbol{\beta}_t)^\top\widetilde{\xvec} &=(\widehat{\gvec}_{t,i} - \boldsymbol{\beta}_t)^\top( \widetilde{\xvec} \pm  \xvec^{\circ} ) \\
		& =(\widehat{\gvec}_{t,i} - \boldsymbol{\beta}_t)^\top\xvec^{\circ}
		+ (\widehat{\gvec}_{t,i} - \boldsymbol{\beta}_t)^\top
		( \widetilde{\xvec} -  \xvec^{\circ} ) \\
		&\le \alpha_i + \| \widetilde{\xvec}- \xvec^{\circ}\|_1  \\
		&\le \alpha_i+  \frac{K}{T},
	\end{align*}
	which holds with probability at least $1-\delta$  and implies that $\widetilde{\mathcal{S}}_t$ is never empty.
	
	To conclude the proof, we notice that:
	\begin{align*}
		\bigcap_{t\in[T]} \widetilde{\mathcal{S}}_t &= \bigcap_{t\in[T]} \left(\Omega \cap {\mathcal{S}}_t^\circ\right)\\
		& = \Omega \cap \left(\bigcap_{t\in[T]} \widetilde{\mathcal{S}}_t\right).
	\end{align*}
	Noticing that, by Equation~\eqref{eq1: solution} and employing the same reasoning above, it holds:
	\[\widetilde{\xvec}\in \bigcap_{t\in[T]} \widetilde{\mathcal{S}}_t,
	\]
	concludes the proof.
\end{proof}

\subsection{Cumulative Regret of \algnamesoft}
\label{proofs:regretsoft}
\begin{lemma}
	\label{lem:soft_regret}
	For any $\delta\in(0,1)$, for any $\uvec$ s.t. $\uvec\in\cap_{t\in[T]}\widetilde{\mathcal{S}}_t$, Algorithm~\ref{alg: bounded violations} guarantees the following bound:
	\begin{equation*}
		\sum_{t=1}^T \widehat{\lvec_t}^\top\left(\xvec_{t}-\uvec\right) \leq \mathcal{O}\left(\frac{K\ln T}{\eta}+\eta\sum_{t=1}^T\ell_t(a_t)\right) - \frac{\boldsymbol h_T^\top\uvec}{10\eta\ln T},
	\end{equation*}
	with probability at least $1-\delta$.
\end{lemma}
\begin{proof}
~We fix  $\uvec$ s.t. $\uvec\in\cap_{t\in[T]}\widetilde{\mathcal{S}}_t$, that is, any possible vector belonging to the intersection between the safe sets built by Algorithm~\ref{alg: bounded violations}. Then, we apply standard OMD with log-barrier results (see~\citep{corralling}) to obtain:
	\begin{equation}
		\sum_{t=1}^T \widehat{\lvec_t}^\top\left(\xvec_{t}-\uvec\right) \leq \sum_{t=1}^T\left(D_{\psi_t}(\uvec,\xvec_t)-D_{\psi_t}(\uvec,\xvec_{t+1})\right)+\sum_{t=1}^T\sum_{a=1}^K \eta_{t,a}x^2_{t}(a)\widehat{\ell_t}^2(a). \label{eq1:lem_soft_reg}
	\end{equation}
	The result above holds since $\widetilde{\mathcal{S}}_t$ is a polytope (thus, convex) for any $t\in[T]$. As $\widetilde{\mathcal{S}}_t$ is included in $\Delta_K$, then the intersection is convex. Moreover, we notice that, by Lemma~\ref{lem:solution}, we have, with probability at least $1-\delta$, that $\cap_{t\in[T]}\widetilde{\mathcal{S}}_t$ is non-empty, by construction of the estimated safe set. Equation~\eqref{eq1:lem_soft_reg} holds under the event mentioned above.
	
	For the last term of Equation~\eqref{eq1:lem_soft_reg}, we simply notice that, for any $t\in[T]$, it holds:
	\begin{align*}
		\eta_{t,a}x^2_{t}(a)\widehat{\ell_t}^2(a)&= \eta_{t,a}x^2_{t}(a)\frac{\ell_t^2(a)}{x^2_{t}(a)}\mathbbm{1}\{a_t=a\}\\ 
		&\leq \eta_{t,a_t}\ell_t^2(a_t)\\
		& \leq \eta_{t,a_t}\ell_t(a_t) \\
		& \leq \eta_{T,a_t}\ell_t(a_t)\\
		&\leq 5\eta\ell_t(a_t),
	\end{align*}
	where the last inequality holds since $\eta_{T,a}=\kappa^{n_a}\eta_{1,a}$, where $n_a$ is the number of times Algorithm~\ref{alg: bounded violations} increases the learning rate for arm $a$, and $\kappa^{n_a}\leq5$.
	
	In the following, we define $h(y)=y-1-\ln y$. Thus we bound the first two term of Equation~\eqref{eq1:lem_soft_reg} as follows:
	\begin{align}
		\sum_{t=1}^T&\left(D_{\psi_t}(\uvec,\xvec_t)-D_{\psi_t}(\uvec,\xvec_{t+1})\right)\nonumber\\&\leq D_{\psi_{1}}(\uvec,\xvec_{1}) + \sum_{t=1}^{T-1}\left(D_{\psi_{t+1}}(\uvec,\xvec_{t+1})-D_{\psi_t}(\uvec,\xvec_{t+1})\right) \label{eq2:lem_soft_reg}\\
		& = \frac{1}{\eta}\sum_{a=1}^K h\left(\frac{u(a)}{x_1(a)}\right) + \sum_{a=1}^K\sum_{t=1}^{T-1}\left(\frac{1}{\eta_{t+1,a}}-\frac{1}{\eta_{t,a}}\right)h\left(\frac{u(a)}{x_{t+1}(a)}\right), \nonumber
	\end{align}
	where Inequality~\eqref{eq2:lem_soft_reg} holds since the Bregman is always greater or equal than zero.
	
	Thus, we focus on the first term, bounding it as follows:
	\begin{align}
		\frac{1}{\eta}\sum_{a=1}^K h\left(\frac{u(a)}{x_1(a)}\right) &= 	\frac{1}{\eta}\sum_{a=1}^K -\ln(Ku(a))  \label{eq3:lem_soft_reg}\\
		&\leq \frac{K\ln T}{\eta},  \label{eq4:lem_soft_reg}
	\end{align}
	where Equation~\eqref{eq3:lem_soft_reg} holds since $\xvec_t$ is initialized uniformly and Inequality~\eqref{eq4:lem_soft_reg} holds since $u(a)\geq1/T$ for all $a\in[K]$.
	
	To bound the final term, we will refer as $t_a$ to the last time step where the learning rate of arm $a$ is increased. Thus, we proceed as follows.
	\begin{align*}
		\left(\frac{1}{\eta_{t_j+1,a}}-\frac{1}{\eta_{t_j,a}}\right)h\left(\frac{u(a)}{x_{t_j+1}(a)}\right)&=\frac{1-\kappa}{k^{n_j}\eta}h\left(\frac{u(a)}{x_{t_j+1}(a)}\right)\\
		&\leq \frac{-h\left(\frac{u(a)}{x_{t_j+1}(a)}\right)}{5\eta\ln T}\\
		&= \frac{-h\left(\frac{u(a) h_{T,a}}{2}\right)}{5\eta\ln T}\\
		&=\frac{\ln\left(\frac{u(a) h_{T,a}}{2}\right)- \frac{u(a) h_{T,a}}{2}+1}{5\eta\ln T}\\
		&\leq \frac{\ln T- \frac{u(a) h_{T,a}}{2}+1}{5\eta\ln T},
	\end{align*}
	where we used that $1-\kappa\leq -\frac{1}{\ln T}$ and $\frac{u(a) h_{T,a}}{2}\leq\frac{1}{x_{t_j+1}(a)}\leq T$.
	
	Combining the previous bounds concludes the proof.
\end{proof}

\regretsoft*
\begin{proof}
	We first decompose the regret as follows:
	\begin{align*}
		R_T&\coloneq \sum_{t=1}^T \lvec_t^\top \left(\xvec_t- \xvec^*\right)\\
		&= \sum_{t=1}^T \widehat\lvec_t^\top \left(\xvec_t- \uvec\right) + \sum_{t=1}^T (\lvec_t-\widehat{\lvec_t})^\top \xvec_t+ \sum_{t=1}^T (\widehat{\lvec_t}-\lvec_t)^\top \uvec+ \sum_{t=1}^T \lvec_t^\top \left(\uvec- \xvec^*\right)\\
		&\leq \sum_{t=1}^T \widehat\lvec_t^\top \left(\xvec_t- \uvec\right) + \sum_{t=1}^T (\lvec_t-\widehat{\lvec_t})^\top \xvec_t+ \sum_{t=1}^T (\widehat{\lvec_t}-\lvec_t)^\top \uvec+ K,
	\end{align*}
	where we take $\uvec$ as $\min_{ \uvec \in \cap_{t\in[T]}\widetilde{\mathcal{S}}_t}\|\xvec^*-\uvec\|_1$ and the inequality follows from the Hölder inequality after noticing that, under the event of Lemma~\ref{lem:solution}, which holds with probability at least $1-\delta$, the maximum $\ell_1$ distance between the safe optimum $\xvec^*$ and $\uvec$ is $K/T$.
	
	We will bound the remaining quantities separately.
	\paragraph{Bound on the first term} The first term follows from Lemma~\ref{lem:soft_regret}.
	
	\paragraph{Bound on the second term} To bound the second term we notice that it is a Martingale difference sequence, where any difference is bounded as:
	\begin{align*}
		\left|\left(\mathbb{E}_t\left[\widehat{\lvec_t}\right]-\widehat{\lvec_t}\right)^\top \xvec_t\right|&\leq \widehat{\lvec_t}^\top \xvec_t\\
		& = \sum_{a=1}^K x_t(a)\frac{\ell_t(a)}{x_t(a)}\mathbbm{1}\{a_t=a\}\\
		&\leq 1.
	\end{align*}
	Furthermore, we bound the second moment as:
	\begin{align*}
		\mathbb{E}_t\left[\left(\left(\mathbb{E}_t\left[\widehat{\lvec_t}\right]-\widehat{\lvec_t}\right)^\top \xvec_t\right)^2\right] & =  	\mathbb{E}_t\left[\left(\left(\lvec_t-\widehat{\lvec_t}\right)^\top \xvec_t\right)^2\right]\\
		& \leq \mathbb{E}_t\left[\left(\widehat\lvec_t^\top\xvec_t\right)^2 \right]\\
		& \leq \mathbb{E}_t\left[\widehat\lvec_t^\top\xvec_t \right] \\
		&=\lvec_t^\top\xvec_t.
	\end{align*}
	Thus we can apply the Freedman inequality to attain, with probability at least $1-\delta$:
	\begin{equation*}
		\sum_{t=1}^T (\lvec_t-\widehat{\lvec_t})^\top \xvec_t = \mathcal{O}\left(\sqrt{\sum_{t=1}^{T}\lvec_t^\top\xvec_t\ln\left(\frac{1}{\delta}\right)}+\ln\left(\frac{1}{\delta}\right)\right).
	\end{equation*}
	
	\paragraph{Bound on the third term} To bound the third term, we again notice that the quantity of interest is a Martingale difference sequence, but we apply a modified version of the Freedman inequality (see~\citep{lee2020bias}). 
	
	First we notice that:
	\begin{equation*}
		(\widehat{\lvec_t}-\lvec_t)^\top \uvec\leq \boldsymbol{ h}_t^\top\uvec \in [1,T].
	\end{equation*}
	We now focus on bounding the second moment as follows:
	\begin{align*}
		\mathbb{E}_t\left[ \left((\widehat{\lvec_t}-\lvec_t)^\top \uvec\right)^2\right] & \leq \mathbb{E}_t\left[ \left(\widehat{\lvec_t}^\top \uvec\right)^2\right]\\
		& = \mathbb{E}_t\left[ \frac{\ell^2_t(a_t) u^2(a_t)}{x^2_t(a_t)}\right]\\
		&\leq \sum_{a=1}^Ku^2(a)\ell_t(a) h_{T,a} \\
		& \leq \boldsymbol{ h}_T^\top\uvec \cdot \lvec_t^\top\uvec.
	\end{align*}
	Thus, with probability at least $1-\delta$, we have by Theorem 2.2 of~\citep{lee2020bias}:
	\begin{equation*}
		\sum_{t=1}^T (\widehat{\lvec_t}-\lvec_t)^\top \uvec = H\left( \sqrt{8\sum_{t=1}^T\lvec_t^\top\uvec\cdot\boldsymbol{ h}_T^\top\uvec \ln\left(\frac{H}{\delta}\right)} + 2\boldsymbol{ h}_T^\top\uvec\ln\left(\frac{H}{\delta}\right)\right),
	\end{equation*}
	where $H=\ln\left(\frac{\lceil\log(T)\rceil \lceil3\log(T)\rceil}{\delta}\right)$.
	
	\paragraph{Final result}
	Combining the previous results and applying a Union Bound, we have, with probability at least $1-3\delta$:
	\begin{align*}
		R_T \leq \mathcal{O}\left(\frac{K\ln T}{\eta}+\eta\sum_{T=1}^T\ell_t(a_t)\right) - \frac{\boldsymbol h_T^\top\uvec}{10\eta\ln T} + \mathcal{O}\left(\sqrt{\sum_{t=1}^{T}\lvec_t^\top\xvec_t\ln\left(\frac{1}{\delta}\right)}+\ln\left(\frac{1}{\delta}\right)\right) + \\ \mkern150mu H\left( \sqrt{8\sum_{t=1}^T\lvec_t^\top\uvec\cdot\boldsymbol{ h}_T^\top\uvec \ln\left(\frac{H}{\delta}\right)} + 2\boldsymbol{ h}_T^\top\uvec\ln\left(\frac{H}{\delta}\right)\right).
	\end{align*}
	Now we notice that, applying Freedman inequality, it is easy to show the following bound:
	\begin{equation*}
		\sum_{T=1}^T\ell_t(a_t) - \sum_{t=1}^T \lvec_t^\top \xvec_t \leq 2\sqrt{\sum_{t=1}^T \lvec_t^\top \xvec_t \ln\left(\frac{1}{\delta}\right)} + \ln\left(\frac{1}{\delta}\right),
	\end{equation*}
	which holds with probability at least $1-\delta$ and implies, by AM-GM inequality:
	\begin{equation*}
		\sum_{t=1}^T\ell_t(a_t) \leq 2\sum_{t=1}^T \lvec_t^\top \xvec_t + 2\ln\left(\frac{1}{\delta}\right).
	\end{equation*}
	Now, going back to the regret bound, it holds, with probability at least $1-4\delta$, by Union Bound:
	\begin{subequations}
		\begin{align}
			R_T&\leq \mathcal{O}\left(\frac{K\ln T}{\eta}+\eta\sum_{t=1}^T\lvec_t^\top\xvec_t + \eta\ln\left(\frac{1}{\delta}\right) \right) - \frac{\boldsymbol h_T^\top\uvec}{10\eta\ln T} + \mathcal{O}\left(\sqrt{\sum_{t=1}^{T}\lvec_t^\top\xvec_t\ln\left(\frac{1}{\delta}\right)}+\ln\left(\frac{1}{\delta}\right)\right)  \nonumber\\ &\mkern150mu + H\left( \sqrt{8\sum_{t=1}^T\lvec_t^\top\uvec\cdot\boldsymbol{ h}_T^\top\uvec \ln\left(\frac{H}{\delta}\right)} + 2\boldsymbol{ h}_T^\top\uvec\ln\left(\frac{H}{\delta}\right)\right)\nonumber\\
			& =\mathcal{O}\left(\frac{K\ln T}{\eta}+\eta\sum_{t=1}^T\lvec_t^\top\xvec_t + \ln\left(\frac{1}{\delta}\right) + \sqrt{\sum_{t=1}^{T}\lvec_t^\top\xvec_t\ln\left(\frac{1}{\delta}\right)}\right) - \frac{\boldsymbol h_T^\top\uvec}{10\eta\ln T}  \nonumber\\ &\mkern150mu +H\left( \sqrt{8\frac{20H\eta\ln T}{20H\eta \ln T}\sum_{t=1}^T\lvec_t^\top\uvec\cdot\boldsymbol{ h}_T^\top\uvec \ln\left(\frac{H}{\delta}\right)} + 2\boldsymbol{ h}_T^\top\uvec\ln\left(\frac{H}{\delta}\right)\right)\nonumber\\
			&\leq \mathcal{O}\left(\frac{K\ln T}{\eta}+\eta\sum_{t=1}^T\lvec_t^\top\xvec_t + \ln\left(\frac{1}{\delta}\right) + \sqrt{\sum_{t=1}^{T}\lvec_t^\top\xvec_t\ln\left(\frac{1}{\delta}\right)}\right) - \frac{\boldsymbol h_T^\top\uvec}{10\eta\ln T}  \nonumber\\ &\mkern75mu +160H^2 \eta\ln T\sum_{t=1}^T\lvec_t^\top\uvec \ln\left(\frac{H}{\delta}\right)  + \frac{H}{20H\eta\ln T}\cdot\boldsymbol{ h}_T^\top\uvec + 2H\boldsymbol{ h}_T^\top\uvec\ln\left(\frac{H}{\delta}\right)\label{eq1:regret_soft}\\
			& \leq \mathcal{O}\left(\frac{K\ln T}{\eta}+\eta\sum_{t=1}^T\lvec_t^\top\xvec_t + \ln\left(\frac{1}{\delta}\right) + \sqrt{\sum_{t=1}^{T}\lvec_t^\top\xvec_t\ln\left(\frac{1}{\delta}\right)}\right) \nonumber\\ & \mkern350mu+ 160H^2 \eta\ln T\sum_{t=1}^T\lvec_t^\top\uvec \ln\left(\frac{H}{\delta}\right) \label{eq2:regret_soft}\\
			& \leq \widetilde{\mathcal{O}}\left(\frac{K}{\eta}+\eta\sum_{t=1}^T\lvec_t^\top\xvec_t + \sqrt{\sum_{t=1}^{T}\lvec_t^\top\xvec_t\ln\left(\frac{1}{\delta}\right)} + \eta\sum_{t=1}^T\lvec_t^\top\uvec \ln\left(\frac{1}{\delta}\right)\right)\nonumber\\
			& \leq \widetilde{\mathcal{O}}\left(\frac{K}{\eta}+\eta\sum_{t=1}^T\lvec_t^\top\xvec_t + \sqrt{\sum_{t=1}^{T}\lvec_t^\top\xvec_t\ln\left(\frac{1}{\delta}\right)} + \eta\sum_{t=1}^T\lvec_t^\top\xvec^* \ln\left(\frac{1}{\delta}\right)\right) \label{eq3:regret_soft},
		\end{align}
	\end{subequations}
	where Inequality~\eqref{eq1:regret_soft} holds by AM-GM inequality, Inequality~\eqref{eq2:regret_soft} holds for $\eta \le \frac{1}{40H\ln T\ln\left(\frac{H}{\delta}\right)}$ and Inequality~\eqref{eq3:regret_soft} holds after noticing that by definition of $\uvec$, $\sum_{t=1}^T\lvec_t^\top\uvec\leq \sum_{t=1}^T\lvec_t^\top\xvec^*+K$. Since $\eta \le \frac{1}{2}$, we have:
	\begin{align*}
		\eta\sum_{t=1}^T\lvec_t^\top\xvec_t \le \frac{1}{2}R_T + \eta\sum_{t=1}^T\lvec_t^\top\xvec^*,
	\end{align*}
	and the regret can be rewritten as:
	\begin{align*}
		R_T \le \widetilde{\mathcal{O}}\left( \frac{2K}{\eta} + 2\sqrt{\left(\frac{1}{2}R_T + \eta\sum_{t=1}^T \lvec_t^\top \xvec^*\right)\ln\left(\frac{1}{\delta}\right)}+ 4\eta\sum_{t=1}^T\lvec_t^\top\xvec^*\right).
	\end{align*}
	We then set $\eta= \min\left\{\frac{1}{40H\ln T\ln\left(\frac{H}{\delta}\right)}, \sqrt{\frac{K}{\sum_{t=1}^{T}\lvec_t^\top\xvec^* \ln\left(\frac{1}{\delta}\right)}}\right\}$ and we solve the quadratic inequality in $R_T$, obtaining
	the following regret bound:
	\begin{align*}
		R_T \le \widetilde{\mathcal{O}}\left( \sqrt{K\sum_{t=1}^{T}\lvec_t^\top\xvec^* \ln\left(\frac{1}{\delta}\right)}\right).
	\end{align*}
	This concludes the proof.
\end{proof}

\subsection{Cumulative Violations of \algnamesoft}
\label{proofs:violation}
\violationsoft*
\begin{proof}
	First, we underline that the following analysis holds for every constraint $i \in [m]$, including the one being violated the most, \emph{i.e.}, $\widetilde{i} \in \argmax_{i \in [m]} \sum_{t=1}^T \left[\boldsymbol{g}_i^\top\xvec_t - \alpha_i\right]^+$. 
	
	By Lemma~\ref{lem:solution} we have that, under the clean event, $\xvec_t \in \widetilde{\mathcal{S}}_t$. By construction, this implies $(\widehat{\boldsymbol{g}}_{t-1,i}-\boldsymbol{\beta}_{t-1})^\top \xvec_t \le \alpha_{{i}} + \frac{K}{T}$ for every $t \in [T]$. Employing Lemma~\ref{lemma:global_hoeffding}, we get, under the clean event:
	\begin{align*}
		\left[\gvec_{{i}}^\top \xvec_t -\alpha_{{i}}\right]^+ \le \frac{K}{T} + 2\boldsymbol{\beta}_{t-1}^\top \xvec_t.
	\end{align*}
	To bound the second term we proceed as follows:
	\begin{align}
		\sum_{t=1}^T\boldsymbol{\beta}_{t-1}^\top \xvec_t &= \sum_{t=1}^T \sum_{a=1}^K \beta_{t-1}(a)x_t(a)  \notag \\
		&\le \sum_{t=1}^T \sum_{a=1}^K \beta_{t-1}(a) \mathbbm{1}_t(a)+ \sqrt{2T\log (\delta^{-1})} \label{eq:00} \\
		&= \sqrt{4\ln\left(\frac{TKm}{\delta}\right)} \sum_{t=1}^T \sum_{a=1}^K  \frac{\mathbbm{1}_t(a)}{\sqrt{\max\{1,N_{t-1}(a)\}}}  + \sqrt{2T\ln (\delta^{-1})} \notag \\
		&\le 3\sqrt{4\ln\left(\frac{TKm}{\delta}\right)}  \sum_{a=1}^K  \sqrt{N_T(a)} + \sqrt{2T\ln (\delta^{-1})} \label{eq:01}\\
		&\le 3\sqrt{4KT\ln\left(\frac{TKm}{\delta}\right)} + \sqrt{2T\ln (\delta^{-1})} \label{eq:02},
	\end{align}
	where Inequality~\eqref{eq:00} follows from Azuma Inequality,with probability at least $1-\delta$, Inequality~\eqref{eq:01} holds since $\sum_{t=1}^T \frac{1}{\sqrt{t}} \le 3\sqrt{T}$, and Inequality~\eqref{eq:02} from Cauchy-Schwarz Inequality and the fact that $\sum_{a=1}^K N_T(a) = T$.
	
	Finally, it holds:
	\begin{align*}
		V_T &= \sum_{t=1}^T \left[\boldsymbol{g}_{{i}}^\top \xvec_t - \alpha_{{i}}\right]^+ \\
		&\le  \sum_{t=1}^T \left( \frac{K}{T} + 2\boldsymbol{\beta}_{t-1}^\top \xvec_t\right) \\
		&\le K + 3\sqrt{4KT\ln\left(\frac{TKm}{\delta}\right)} + \sqrt{2T\ln(\delta^{-1})}.
	\end{align*}
	Employing a Union Bound concludes the proof.
\end{proof}
\section{Omitted Proofs for Hard Constraints}
\label{App:hard}

\subsection{Cumulative Regret of \algnamehard}
\label{proofs:regrethard}
\firstpart*
\begin{proof}
	We first split the round in two sets $T_1,T_2$. $T_1$ encompasses the rounds $t\in[T]$ s.t. $\gamma_{t-1}\leq 1/2$, $T_2$ the remaining rounds.
	
	\paragraph{Bound in $T_1$} 
	We apply the Cauchy–Schwarz inequality obtaining the following bound:
	\begin{align*}
		\sum_{t\in T_1} \gamma_{t-1} \lvec_t^\top(\xvec^\diamond - \xvec^*) & \leq \sqrt{\left(\sum_{t\in T_1} \gamma^2_{t-1}\right) \left(\sum_{t\in T_1}\left(\lvec_t^\top(\xvec^\diamond - \xvec^*)\right)^2\right)} \\ 
		& =\sqrt{\sum_{t\in T_1} \gamma^2_{t-1}}\cdot \sqrt{\sum_{t=1}^T\left(\lvec_t^\top(\xvec^\diamond - \xvec^*)\right)^2}. %\\
		%& \leq   \sqrt{\sum_{t\in T_1} \gamma^2_{t-1}}\cdot \sqrt{\sum_{t=1}^T\lvec_t^\top(\xvec^\diamond - \xvec^*)} .
	\end{align*}
	We will now focus on the bounding the sequence $\sum_{t\in T_1}\gamma^2_{t-1}$.
	
	We proceed as follows:
	\begin{align*}
		\sum_{t\in T_1} \gamma^2_{t-1} &= \sum_{t\in T_1}\max_{  i\in[m]}\left\{\frac{\min\{(\widehat{\gvec}_{t-1,i}+\boldsymbol{\beta}_{t-1})^\top \widetilde\xvec_{t},1\}-\alpha_i}{\min\{(\widehat{\gvec}_{t-1,i}+\boldsymbol{\beta}_{t-1})^\top \widetilde\xvec_{t},1\}-\theta_i}\right\}^2 \\
		& \leq \sum_{t\in T_1}\max_{  i\in[m]}\left\{\frac{(\widehat{\gvec}_{t-1,i}+\boldsymbol{\beta}_{t-1})^\top \widetilde\xvec_{t}-\alpha_i}{(\widehat{\gvec}_{t-1,i}+\boldsymbol{\beta}_{t-1})^\top \widetilde\xvec_{t}-\theta_i}\right\}^2 \\
		& \leq \sum_{t\in T_1} \left(\frac{2\boldsymbol{\beta}_{t-1}^\top \widetilde{\xvec}_t+ \frac{K}{T}}{\rho}\right)^2\\
		&\leq \sum_{t\in T_1} 2\left(\frac{2\boldsymbol{\beta}_{t-1}^\top \widetilde{\xvec}_t}{\rho}\right)^2 + \frac{2K^2}{\rho^2} ,
	\end{align*}
	where the second inequality holds by definition of $\rho$ and by Lemma~\ref{lemma:global_hoeffding}, with probability at least $1-\delta$.
	
	Thus we bound the following quantity:
	\begin{align*}
		\sum_{t\in T_1}  \left(\boldsymbol{\beta}_{t-1}^\top \widetilde{\xvec}_t\right)^2 & \leq \sum_{t\in T_1} \left(2(1-\gamma_{t-1})\boldsymbol{\beta}_{t-1}^\top \widetilde{\xvec}_t\right)^2 \\
		& \leq 4\sum_{t=1}^T \left(\boldsymbol{\beta}_{t-1}^\top {\xvec}_t\right)^2 \\
		& = 4\sum_{t=1}^T \left(\sum_{a=1}^K\sqrt{\frac{4\ln(TKm/\delta)}{\max\{1,N_{t-1}(a)\}}}x_t(a)\right)^2 \nonumber\\
		& \leq 4K \sum_{t=1}^T \sum_{a=1}^K\left(\sqrt{\frac{4\ln(TKm/\delta)}{\max\{1,N_{t-1}(a)\}}}x_t(a)\right)^2 \nonumber \\
		& \leq K16\ln(TKm/\delta) \sum_{t=1}^T \sum_{a=1}^K\frac{1}{\max\{1,N_{t-1}(a)\}}x_t(a) \nonumber\\
		&= K16\ln(TKm/\delta) \sum_{t=1}^T \sum_{a=1}^K\left( \frac{x_t(a)-\mathbbm{1}_t(a)}{\max\{1,N_{t-1}(a)\}} + \frac{\mathbbm{1}_t(a)}{\max\{1,N_{t-1}(a)\}}\right),
	\end{align*}
	where the first step holds since $\gamma_{t-1}\leq1/2$.
	
	We can bound the second term as:
	\begin{equation*}
		\sum_{t=1}^T\sum_{a=1}^K\frac{\mathbbm{1}_t(a)}{\max\{1,N_{t-1}(a)\}} \leq K\left(1+\sum_{t=1}^T\frac{1}{t}\right) \leq 3K+ 2K\ln(T).
	\end{equation*}
	To bound the first term we notice that it is a martingale difference sequence in which any martingale difference is bounded by $1$. Thus, we proceed as follows:
	\begin{align*}
		\mathbb{E}_t\left[\left(\sum_{a=1}^K\frac{x_t(a)-\mathbbm{1}_t(a)}{\max\{1,N_{t-1}(a)\}}\right)^2\right]
		& \leq \mathbb{E}_t\left[K\sum_{a=1}^K\left(\frac{x_t(a)-\mathbbm{1}_t(a)}{\max\{1,N_{t-1}(a)\}}\right)^2\right]  \\
		&= K\sum_{a=1}^K \frac{\mathbb{E}_t\left[(x_t(a)-\mathbbm{1}_t(a))^2\right]}{\max\{1,N^2_{t-1}(a)\}}\\
		& = K\sum_{a=1}^K \frac{x_t(a)(1-x_t(a))}{\max\{1,N^2_{t-1}(a)\}}\\
		&\leq  K\sum_{a=1}^K \frac{x_t(a)}{\max\{1,N_{t-1}(a)\}},\\
	\end{align*}
	and we apply Lemma~9 of~\citep{JinLearningAdversarial2019} with $\lambda=1/2K$ to obtain, with probability at least $1-\delta$:
	\begin{equation*}
		\sum_{t=1}^T\sum_{a=1}^K\frac{x_t(a)-\mathbbm{1}_t(a)}{\max\{1,N_{t-1}(a)\}} \leq \frac{1}{2}\sum_{t=1}^T\sum_{a=1}^K \frac{x_t(a)}{\max\{1,N_{t-1}(a)\}} + 2K\ln(1/\delta).
	\end{equation*}
	Thus, employing a Union Bound, we obtain, with probability at least $1-2\delta$:
	\begin{equation*}
		\sum_{t\in T_1}  \left(\boldsymbol{\beta}_{t-1}^\top \widetilde{\xvec}_t\right)^2\leq 96K^2\ln(TKm/\delta) + 128K^2\ln^2(TKm/\delta),
	\end{equation*}
	and similarly:
	\begin{equation*}
		\sum_{t\in T_1} \gamma^2_{t-1}\leq \frac{768}{\rho^2}K^2\ln(TKm/\delta) + \frac{1024}{\rho^2}K^2\ln^2(TKm/\delta) + \frac{2K^2}{\rho^2}.
	\end{equation*}
	To conclude, we have, with probability at least $1-2\delta$:
	\begin{align*}
		\sum_{t\in T_1} \gamma_{t-1} & \lvec_t^\top(\xvec^\diamond - \xvec^*) \\ & \leq   \sqrt{\sum_{t\in T_1} \gamma^2_{t-1}}\cdot \sqrt{\sum_{t=1}^T\left(\lvec_t^\top(\xvec^\diamond - \xvec^*)\right)^2} \\
		& \leq \sqrt{\frac{768}{\rho^2}K^2\ln(TKm/\delta) + \frac{1024}{\rho^2}K^2\ln^2(TKm/\delta) + \frac{2K^2}{\rho^2}}\cdot \sqrt{\sum_{t=1}^T\left(\lvec_t^\top(\xvec^\diamond - \xvec^*)\right)^2} \\
		& \leq \frac{43K\ln(TKm/\delta)}{\rho} \sqrt{\sum_{t=1}^T\left(\lvec_t^\top(\xvec^\diamond - \xvec^*)\right)^2}.
	\end{align*}

	\paragraph{Bound in $T_2$}
	We first apply the Hölder inequality to obtain the following bound:
	\begin{align*}
		\sum_{t\in T_2} \gamma_{t-1} \lvec_t^\top(\xvec^\diamond - \xvec^*) & \leq \sum_{t\in T_2} \gamma_{t-1}.
	\end{align*}
	To bound the aforementioned terms, we upper bound the cardinality of the set $T_2$. This is done by first bounding the cardinality of the following set:
	\begin{equation*}
		T_3 = \left\{t\in[T]: \sum_{a=1}^K \beta_{t-1}(a)\mathbbm{1}_t(a)\geq \frac{\rho^2}{8}\right\}.
	\end{equation*}
	From the definition we can state the following lower bound:
	\begin{equation*}
		\sum_{t\in T_3} \sum_{a=1}^K \beta_{t-1}(a)\mathbbm{1}_t(a) \geq |T_3|\frac{\rho^2}{8}.
	\end{equation*}
	We first bound the quantity $\sum_{t\in T_3} \sum_{a=1}^K \beta_{t-1}(a)\mathbbm{1}_t(a)$ similarly to what done in Theorem~\ref{thm:viol} as:
	\begin{equation*}
		\sum_{t\in T_3} \sum_{a=1}^K \beta_{t-1}(a)\mathbbm{1}_t(a)\leq 3\sqrt{4K|T_3|\ln\left(\frac{TKm}{\delta}\right)},
	\end{equation*}
	which holds with probability at least $1-\delta$.
	Combining the previous bounds, we obtain:
	\begin{equation*}
		|T_3|\frac{\rho^2}{8} \leq 3\sqrt{4K|T_3|\ln\left(\frac{TKm}{\delta}\right)},
	\end{equation*}
	which implies:
	\begin{equation*}
		|T_3|\leq \frac{2304}{\rho^4}K\ln\left(\frac{TKm}{\delta}\right).
	\end{equation*}
	Thus, we employ the reverse Markov inequality to bound the probability that $t\in T_2\cap T_3$. First we lower bound the following quantity:
	\begin{align*}
		\mathbb{E}_t \left[\sum_{a=1}^K \beta_{t-1}(a)\mathbbm{1}_t(a)\right] & = \boldsymbol{\beta}_{t-1}^\top \xvec_{t} \\
		& \geq (1-\gamma_{t-1})\boldsymbol{\beta}_{t-1}^\top\widetilde{\xvec}_t\\
		& \geq \frac{\rho}{1+\rho}\boldsymbol{\beta}_{t-1}^\top\widetilde{\xvec}_t\\
		& \geq \frac{\rho^2}{4}-\frac{\rho K}{T}\\
		&\geq \frac{\rho^2}{6},
	\end{align*}
	where the last steps hold since:
	\begin{align*}
		\gamma_{t-1}\leq\max_{i\in[m]}\left\{\frac{1-\alpha_i}{1-\theta_i} \right\}
		= \max_{i\in[m]}\left\{\frac{1-\alpha_i}{1+\rho - \alpha_i}\right\} 
		\leq \frac{1}{1+\rho},
	\end{align*} 
	$\boldsymbol{\beta}_{t-1}^\top\widetilde{\xvec}_t\geq \frac{\rho}{2}-\frac{K}{T}$ when $t\in T_2$, under the clean event and for $\rho\geq\frac{12K}{T}$. We can now employ the reverse Markov inequality to state:
	\begin{equation*}
		\mathbb{P}\left\{\sum_{a=1}^K \beta_{t-1}(a)\mathbbm{1}_t(a)\geq \frac{\rho^2}{8}| \mathcal{F}_{t-1}\right\}\geq \frac{\frac{\rho^2}{6}- \frac{\rho^2}{8}}{1-\frac{\rho^2}{8}} \geq \frac{\rho^2}{24}.
	\end{equation*}
	Employing the equation above we can state that:
	\begin{equation*}
		\frac{2304}{\rho^4}K\ln\left(\frac{TKm}{\delta}\right)\geq 	|T_3| \geq \frac{\rho^2}{24} |T_2|,
	\end{equation*}
	from which:
	\begin{equation*}
		|T_2|\leq \frac{55296}{\rho^6} K\ln\left(\frac{TKm}{\delta}\right).
	\end{equation*}
	To conclude, we have, with probability at least $1-\delta$, the following bound:
	\begin{equation*}
		\sum_{t\in T_2} \gamma_{t-1} \lvec_t^\top(\xvec^\diamond - \xvec^*) \leq \sum_{t\in T_2} \gamma_{t-1}\leq |T_2| \leq \frac{55296}{\rho^6} K\ln\left(\frac{TKm}{\delta}\right).
	\end{equation*}
	
	\paragraph{Combining everything}
	
	Considering the quantity of interest, we have the following bound with probability at least $1-2\delta$ by Union Bound:
	\begin{align*}
		\sum_{t=1}^T \gamma_{t-1} \lvec_t^\top(\xvec^\diamond - \xvec^*)& = \sum_{t\in T_1} \gamma_{t-1} \lvec_t^\top(\xvec^\diamond - \xvec^*) + \sum_{t\in T_2} \gamma_{t-1} \lvec_t^\top(\xvec^\diamond - \xvec^*)\\
		& \leq \frac{43K\ln(TKm/\delta)}{\rho} \sqrt{\sum_{t=1}^T\left(\lvec_t^\top(\xvec^\diamond - \xvec^*)\right)^2}+ \frac{55296}{\rho^6} K\ln\left(\frac{TKm}{\delta}\right).\\
		%& =\frac{31K\ln(TKm/\delta)}{\rho} \sqrt{\sum_{t=1}^T\left(\lvec_t^\top(\xvec^\diamond - \xvec^*)\right)^2}+ \frac{18432}{\rho^6} K\ln\left(\frac{TKm}{\delta}\right) .
	\end{align*}
	This concludes the proof.	
\end{proof}

\secondpart*
\begin{proof}
	Similarly to the analysis employed to prove Theorem~\ref{thm:reg_soft}, we decompose the quantity of interest as follows:
	\begin{align*}
		&\sum_{t=1}^T (1-\gamma_{t-1}) \lvec_t^\top(\widetilde{\xvec}_t - \xvec^*)\\
		&\leq \sum_{t=1}^T(1-\gamma_{t-1})\widehat\lvec_t^\top \left(\widetilde{\xvec}_t - \uvec\right) + \sum_{t=1}^T(1-\gamma_{t-1}) (\lvec_t-\widehat{\lvec_t})^\top \widetilde{\xvec}_t + \sum_{t=1}^T(1-\gamma_{t-1}) (\widehat{\lvec_t}-\lvec_t)^\top \uvec\\ & \mkern550mu+ \sum_{t=1}^T \lvec_t^\top \left(\uvec- \xvec^*\right)\\
		&\leq \sum_{t=1}^T(1-\gamma_{t-1}) \widehat\lvec_t^\top \left(\widetilde{\xvec}_t - \uvec\right) + \sum_{t=1}^T(1-\gamma_{t-1}) (\lvec_t-\widehat{\lvec_t})^\top \widetilde{\xvec}_t + \sum_{t=1}^T(1-\gamma_{t-1}) (\widehat{\lvec_t}-\lvec_t)^\top \uvec+ K.
	\end{align*}
	We proceed bounding each term separately.
	
	\paragraph{Bound on the first term} 
	
	To bound the first term, we can apply a similar analysis to the one of Lemma~\ref{lem:soft_regret}, since, $\widetilde\xvec_{t}$ is played independently on $\gamma_{t-1}$ except for the loss estimator, to attain, under the clean event:
	\begin{align*}
		\sum_{t=1}^T(1-\gamma_{t-1}) &\widehat\lvec_t^\top \left(\widetilde{\xvec}_t - \uvec\right)\\&\leq \mathcal{O}\left(\frac{K\ln T}{\eta}\right) - \min_{t\in[T]}(1-\gamma_{t-1}) \frac{\boldsymbol h_T^\top\uvec}{10\eta\ln T} + \sum_{t=1}^T(1-\gamma_{t-1})\sum_{a=1}^K \eta_{t,a}\widetilde x^2_{t}(a)\widehat{\ell_t}^2(a)\\
		&\leq \mathcal{O}\left(\frac{K\ln T}{\eta}\right) - \rho \frac{\boldsymbol h_T^\top\uvec}{10\eta\ln T} + \sum_{t=1}^T(1-\gamma_{t-1})\sum_{a=1}^K \eta_{t,a}\widetilde x^2_{t}(a)\widehat{\ell_t}^2(a).
	\end{align*}
	To bound the last term, we proceed as follows:
	\begin{align*}
		(1-\gamma_{t-1})\eta_{t,a}\widetilde x^2_{t}(a)\widehat{\ell_t}^2(a)&= 	(1-\gamma_{t-1})\eta_{t,a}\widetilde x^2_{t}(a)\frac{\ell_t^2(a)}{x^2_{t}(a)}\mathbbm{1}\{a_t=a\}\\ 
		& = \frac{1}{1-\gamma_{t-1}} (1-\gamma_{t-1})^2\eta_{t,a}\widetilde x^2_{t}(a)\frac{\ell_t^2(a)}{x^2_{t}(a)}\mathbbm{1}\{a_t=a\} \\
		&\leq \frac{1}{\rho}\eta_{t,a} x^2_{t}(a)\frac{\ell_t^2(a)}{x^2_{t}(a)}\mathbbm{1}\{a_t=a\} \\
		&\leq \frac{1}{\rho}\eta_{t,a_t}\ell_t^2(a_t)\\
		& \leq \frac{1}{\rho}\eta_{t,a_t}\ell_t(a_t) \\
		& \leq \frac{1}{\rho}\eta_{T,a_t}\ell_t(a_t)\\
		&\leq \frac{5}{\rho}\eta\ell_t(a_t).
	\end{align*}
	Thus, we obtain the following final bound, which holds with probability at least $1-\delta$:
	\begin{equation*}
		\sum_{t=1}^T(1-\gamma_{t-1}) \widehat\lvec_t^\top \left(\widetilde{\xvec}_t - \uvec\right)\leq \mathcal{O}\left(\frac{K\ln T}{\eta}+\frac{\eta}{\rho}\sum_{t=1}^T\ell_t(a_t)\right) - \rho\frac{\boldsymbol h_T^\top\uvec}{10\eta\ln T}.
	\end{equation*}
	\paragraph{Bound on the second term} To bound the second term we first notice that $(1-\gamma_{t-1})\widetilde{\xvec}_t\leq \xvec_t$. Thus, we proceed similarly to~Theorem~\ref{thm:reg_soft}, noticing that the quantity of interest is Martingale difference sequence, where any difference is bounded as:
	\begin{align*}
		\left|(1-\gamma_{t-1})\left(\mathbb{E}_t\left[\widehat{\lvec_t}\right]-\widehat{\lvec_t}\right)^\top\widetilde{ \xvec}_t\right|&\leq (1-\gamma_{t-1})\widehat{\lvec_t}^\top \widetilde\xvec_t\\
		& \leq \widehat{\lvec_t}^\top \xvec_t\\
		& = \sum_{a=1}^K x_t(a)\frac{\ell_t(a)}{x_t(a)}\mathbbm{1}\{a_t=a\}\\
		&\leq 1.
	\end{align*}
	Furthermore, we bound the second moment as:
	\begin{align*}
		\mathbb{E}_t\left[\left((1-\gamma_{t-1})\left(\mathbb{E}_t\left[\widehat{\lvec_t}\right]-\widehat{\lvec_t}\right)^\top \widetilde\xvec_t\right)^2\right] & =  	\mathbb{E}_t\left[\left((1-\gamma_{t-1})\left(\lvec_t-\widehat{\lvec_t}\right)^\top \widetilde\xvec_t\right)^2\right]\\
		& \leq \mathbb{E}_t\left[\left((1-\gamma_{t-1})\widehat\lvec_t^\top\widetilde\xvec_t\right)^2 \right]\\
		& \leq \mathbb{E}_t\left[\left(\widehat\lvec_t^\top\xvec_t\right)^2 \right]\\
		& \leq \mathbb{E}_t\left[\widehat\lvec_t^\top\xvec_t \right] \\
		&=\lvec_t^\top\xvec_t.
	\end{align*}
	Thus we can apply the Freedman inequality to attain, with probability at least $1-\delta$:
	\begin{equation*}
		\sum_{t=1}^T(1-\gamma_{t-1}) (\lvec_t-\widehat{\lvec_t})^\top \widetilde{\xvec}_t \leq  \mathcal{O}\left(\sqrt{\sum_{t=1}^{T}\lvec_t^\top\xvec_t\ln\left(\frac{1}{\delta}\right)}+\ln\left(\frac{1}{\delta}\right)\right).
	\end{equation*}
	\paragraph{Bound on the third term}
	To bound the third term, we notice that the quantity of interest is a Martingale difference sequence. To apply the modified version of the Freedman inequality (see~\citep{lee2020bias}), we notice that:
	\begin{align*}
		(1-\gamma_{t-1})(\widehat{\lvec_t}-\lvec_t)^\top \uvec &\leq  
		(1-\gamma_{t-1})\sum_{a=1}^K\frac{1}{{x}_t(a)}u(a)\\
		& = (1-\gamma_{t-1})\sum_{a=1}^K\frac{1}{\gamma_{t-1}x^\diamond(a)+(1-\gamma_{t-1})\widetilde{x}_t(a)}u(a)\\
		& \leq (1-\gamma_{t-1})\sum_{a=1}^K\frac{1}{(1-\gamma_{t-1})\widetilde{x}_t(a)}u(a)\\
		& \leq \sum_{a=1}^K\frac{1}{\min_{\tau\in[t]}\widetilde{x}_\tau(a)}u(a)\\ & \leq \boldsymbol{h}_t^\top\uvec \in [1,T].
	\end{align*}
	We now focus on bounding the second moment as follows:
	\begin{align*}
		\mathbb{E}_t\left[ \left((1-\gamma_{t-1})(\widehat{\lvec_t}-\lvec_t)^\top \uvec\right)^2\right] & \leq \mathbb{E}_t\left[ \left((1-\gamma_{t-1})\widehat{\lvec_t}^\top \uvec\right)^2\right]\\
		& = \mathbb{E}_t\left[ (1-\gamma_{t-1})^2\frac{\ell^2_t(a_t) u^2(a_t)}{x^2_t(a_t)}\right]\\
		& \leq \mathbb{E}_t\left[ (1-\gamma_{t-1})^2\frac{\ell^2_t(a_t) u^2(a_t)}{(1-\gamma_t)^2\widetilde x^2_t(a_t)}\right]\\
		&\leq \sum_{a=1}^Ku^2(a)\ell_t(a) h_{T,a} \\
		& \leq \boldsymbol{ h}_T^\top\uvec \cdot \lvec_t^\top\uvec.
	\end{align*}
	Thus, with probability at least $1-\delta$, we have by Theorem 2.2 of~\citep{lee2020bias}:
	\begin{equation*}
		\sum_{t=1}^T (1-\gamma_{t-1})(\widehat{\lvec_t}-\lvec_t)^\top \uvec = H\left( \sqrt{8\sum_{t=1}^T\lvec_t^\top\uvec\cdot\boldsymbol{ h}_T^\top\uvec \ln\left(\frac{H}{\delta}\right)} + \boldsymbol{ h}_T^\top\uvec\ln\left(\frac{H}{\delta}\right)\right),
	\end{equation*}
	where $H=\ln\left(\frac{\lceil\ln(T)\rceil \lceil3\ln(T)\rceil}{\delta}\right)$.
	\paragraph{Final result} Combining the previous equations, we get, with probability at least $1-3\delta$, by Union Bound, the following bound:
	\begin{align*}
		\sum_{t=1}^T &(1-\gamma_{t-1})\lvec_t^\top(\widetilde{\xvec}_t - \xvec^*)  \\ &\leq  \mathcal{O}\left(\frac{K\ln T}{\eta}+\frac{\eta}{\rho}\sum_{t=1}^T\ell_t(a_t)\right) - \rho\frac{\boldsymbol h_T^\top\uvec}{10\eta\ln T}+  \mathcal{O}\left(\sqrt{\sum_{t=1}^{T}\lvec_t^\top\xvec_t\ln\left(\frac{1}{\delta}\right)}+\ln\left(\frac{1}{\delta}\right)\right)\\&\mkern10mu+ H\left( \sqrt{8\sum_{t=1}^T\lvec_t^\top\uvec\cdot\boldsymbol{ h}_T^\top\uvec \ln\left(\frac{H}{\delta}\right)} + \boldsymbol{ h}_T^\top\uvec\ln\left(\frac{H}{\delta}\right)\right) + K\\
		&\leq \mathcal{O}\left(\frac{K\ln T}{\eta}+\frac{\eta}{\rho}\sum_{t=1}^T\ell_t(a_t)+\sqrt{\sum_{t=1}^{T}\lvec_t^\top\xvec_t\ln\left(\frac{1}{\delta}\right)}\right) -\rho \frac{\boldsymbol h_T^\top\uvec}{10\eta\ln T}\\&\mkern10mu+ H\left( \sqrt{8\sum_{t=1}^T\lvec_t^\top\uvec\cdot\boldsymbol{ h}_T^\top\uvec \ln\left(\frac{H}{\delta}\right)} + \boldsymbol{ h}_T^\top\uvec\ln\left(\frac{H}{\delta}\right)\right),\\
	\end{align*}
	where $H=\ln\left(\frac{\lceil\ln(T)\rceil \lceil3\ln(T)\rceil}{\delta}\right)$.
	
	Finally, we proceed similarly to Theorem~\ref{thm:reg_soft}, obtaining:
	\begin{align*}
		\sum_{t=1}^T& (1-\gamma_{t-1}) \lvec_t^\top(\widetilde{\xvec}_t - \xvec^*)  \\ 
		&\leq \mathcal{O}\left(\frac{K\ln T}{\eta}+\frac{2\eta}{\rho}\sum_{t=1}^T\lvec_t^\top\xvec_t +\frac{2\eta}{\rho}\ln\left(\frac{1}{\delta}\right)+\sqrt{\sum_{t=1}^{T}\lvec_t^\top\xvec_t\ln\left(\frac{1}{\delta}\right)}\right) - \rho \frac{\boldsymbol h_T^\top\uvec}{10\eta\ln T}\\&\mkern10mu+ H\left( \sqrt{8\sum_{t=1}^T\lvec_t^\top\uvec\cdot\boldsymbol{ h}_T^\top\uvec \ln\left(\frac{H}{\delta}\right)} + \boldsymbol{ h}_T^\top\uvec\ln\left(\frac{H}{\delta}\right)\right)\\
		&= \mathcal{O}\left(\frac{K\ln T}{\eta}+\frac{2\eta}{\rho}\sum_{t=1}^T\lvec_t^\top\xvec_t +\frac{2\eta}{\rho}\ln\left(\frac{1}{\delta}\right)+\sqrt{\sum_{t=1}^{T}\lvec_t^\top\xvec_t\ln\left(\frac{1}{\delta}\right)}\right) - \rho \frac{\boldsymbol h_T^\top\uvec}{10\eta\ln T}\\&\mkern10mu+ H\left( \sqrt{8\frac{20\rho H\eta\ln T}{20\rho H\eta \ln T}\sum_{t=1}^T\lvec_t^\top\uvec\cdot\boldsymbol{ h}_T^\top\uvec \ln\left(\frac{H}{\delta}\right)} + \boldsymbol{ h}_T^\top\uvec\ln\left(\frac{H}{\delta}\right)\right)\\
		&\leq \mathcal{O}\left(\frac{K\ln T}{\eta}+\frac{2\eta}{\rho}\sum_{t=1}^T\lvec_t^\top\xvec_t +\frac{2\eta}{\rho}\ln\left(\frac{1}{\delta}\right)+\sqrt{\sum_{t=1}^{T}\lvec_t^\top\xvec_t\ln\left(\frac{1}{\delta}\right)}\right) - \rho \frac{\boldsymbol h_T^\top\uvec}{10\eta\ln T}\\&\mkern10mu+ \frac{160H^2\eta\ln T}{\rho }\sum_{t=1}^T\lvec_t^\top\uvec\ln\left(\frac{H}{\delta}\right)+\frac{\rho H}{20H\eta\ln T}\boldsymbol{ h}_T^\top\uvec + H\boldsymbol{ h}_T^\top\uvec\ln\left(\frac{H}{\delta}\right)\\
		& \leq \widetilde{\mathcal{O}}\left(\frac{K}{\eta}+ \frac{\eta}{\rho}\sum_{t=1}^T\lvec_t^\top\xvec_t+ \frac{K\eta}{\rho}\ln\left(\frac{1}{\delta}\right)+\sqrt{\sum_{t=1}^{T}\lvec_t^\top\xvec_t\ln\left(\frac{1}{\delta}\right)}+\frac{\eta}{\rho}\sum_{t=1}^T\lvec_t^\top\xvec^*\ln\left(\frac{1}{\delta}\right)\right),
	\end{align*}
	where the first step holds by Freedman inequality and a union bound, setting the confidence to $1-4\delta$, and AM-GM inequality, the third step by AM-GM inequality and the last step holds for $\eta \le \frac{\rho}{40H\ln T\ln\left(\frac{H}{\delta}\right)}$.
	
	This concludes the proof.
\end{proof}

\regrethard*
\begin{proof}
	We first notice that the regret can be decomposed as:
	\begin{align*}
		R_T&\coloneq \sum_{t=1}^T \lvec_t^\top \xvec_t-\lvec_t^\top \xvec^*\\
		& = \sum_{t=1}^T \gamma_{t-1} \lvec_t^\top(\xvec^\diamond - \xvec^*) +  \sum_{t=1}^T (1-\gamma_{t-1}) \lvec_t^\top(\widetilde{\xvec}_t - \xvec^*).
	\end{align*}
	Employing Lemma~\ref{regret:first_part}, Lemma~\ref{regret:second_part} and a Union Bound, we have, with probability at least $1-5\delta$:
	\begin{align*}
		R_T&\leq \frac{43K\ln(TKm/\delta)}{\rho} \sqrt{\sum_{t=1}^T\left(\lvec_t^\top(\xvec^\diamond - \xvec^*)\right)^2}+ \frac{55296}{\rho^6} K\ln\left(\frac{TKm}{\delta}\right) \\ & \mkern10mu+ \widetilde{\mathcal{O}}\left(\frac{K}{\eta}+ \frac{\eta}{\rho}\sum_{t=1}^T\lvec_t^\top\xvec_t+ \frac{K\eta}{\rho}\ln\left(\frac{1}{\delta}\right)+\sqrt{\sum_{t=1}^{T}\lvec_t^\top\xvec_t\ln\left(\frac{1}{\delta}\right)}+\frac{\eta}{\rho}\sum_{t=1}^T\lvec_t^\top\xvec^*\ln\left(\frac{1}{\delta}\right)\right)\\
		& = \widetilde{\mathcal{O}}\Bigg(\frac{K}{\rho} \sqrt{\sum_{t=1}^T\left(\lvec_t^\top(\xvec^\diamond - \xvec^*)\right)^2}+\frac{K}{\eta}+ \frac{\eta}{\rho}\sum_{t=1}^T\lvec_t^\top\xvec_t+ \frac{K\eta}{\rho}\ln\left(\frac{1}{\delta}\right)+\sqrt{\sum_{t=1}^{T}\lvec_t^\top\xvec_t\ln\left(\frac{1}{\delta}\right)}\\&\mkern10mu+\frac{\eta}{\rho}\sum_{t=1}^T\lvec_t^\top\xvec^*\ln\left(\frac{1}{\delta}\right)\Bigg).
	\end{align*}
	Since $\eta \le \frac{\rho}{2}$, we have:
	\begin{align*}
		\frac{\eta}{\rho}\sum_{t=1}^T\lvec_t^\top\xvec_t \le \frac{1}{2}R_T + \frac{\eta}{\rho}\sum_{t=1}^T\lvec_t^\top\xvec^*,
	\end{align*}
	and the regret can be rewritten as:
	\begin{align*}
		R_T \le \widetilde{\mathcal{O}}\Bigg( \frac{K}{\rho} & \sqrt{\sum_{t=1}^T\left(\lvec_t^\top(\xvec^\diamond - \xvec^*)\right)^2}+\frac{2K}{\eta} \\& + 2\sqrt{\left(\frac{1}{2}R_T + \frac{\eta}{\rho}\sum_{t=1}^T\lvec_t^\top\xvec^*\right)\ln\left(\frac{1}{\delta}\right)}+ \frac{4\eta}{\rho}\sum_{t=1}^T\lvec_t^\top\xvec^*\ln\left(\frac{1}{\delta}\right)\Bigg).
	\end{align*}
	We then set $\eta= \min\left\{\frac{\rho}{40H\ln T\ln\left(\frac{H}{\delta}\right)}, \sqrt{\frac{K}{\sum_{t=1}^{T}\lvec_t^\top\xvec^* \ln\left(\frac{1}{\delta}\right)}}\right\}$ and we solve the quadratic inequality in $R_T$, obtaining
	the following regret bound:
	\begin{align*}
		R_T \le \widetilde{\mathcal{O}}\left(  \frac{K}{\rho} \sqrt{\sum_{t=1}^T\left(\lvec_t^\top(\xvec^\diamond - \xvec^*)\right)^2}+\frac{1}{\rho}\sqrt{K\sum_{t=1}^{T}\lvec_t^\top\xvec^* \ln\left(\frac{1}{\delta}\right)}\right).
	\end{align*}
	This concludes the proof.
\end{proof}

\subsection{Safety Property of \algnamehard}
\label{proofs:safety}
\safety*
\begin{proof}
	To prove the result, we consider separately the case in which $\gamma_t=0$ and $\gamma_t\in(0,1)$. Notice that, when $\gamma_t=1$, the constraints are trivially satisfied by definition of the strictly feasible solution.
	
	When $\gamma_t=0$, by construction, it holds that $\forall i \in[m]: (\widehat{\gvec}_{t,i}+\boldsymbol{\beta}_t)^\top \widetilde{\xvec}_{t+1}\leq \alpha_i$ and $\xvec_{t+1}=\widetilde{\xvec}_{t+1}$. Thus, we have:
	\begin{align*}
		\alpha_i &\geq (\widehat{\gvec}_{t,i}+\boldsymbol{\beta}_t)^\top \widetilde{\xvec}_{t+1} \\
		& = (\widehat{\gvec}_{t,i}+\boldsymbol{\beta}_t)^\top {\xvec}_{t+1} \\
		& \geq \gvec_i^\top \xvec_{t+1},
	\end{align*}
	where the last step holds thank to Lemma~\ref{lemma:global_hoeffding} with probability at least $1-\delta$.
	
	When $\gamma_t=(0,1)$, $\gamma_t$ can be selected either as $\frac{1-\alpha_i}{1-\theta_i}$ or as $\frac{(\widehat{\gvec}_{t,i}+\boldsymbol{\beta}_t)^\top \widetilde\xvec_{t+1}-\alpha_i}{(\widehat{\gvec}_{t,i}+\boldsymbol{\beta}_t)^\top \widetilde\xvec_{t+1}-\theta_i}$. In the first case, it holds:
	\begin{align*}
		\gvec_i^\top\xvec_{t+1}&= 	\gvec_i^\top (\gamma_t\xvec^\diamond+ (1-\gamma_t)\widetilde{\xvec}_{t+1}) \\
		&\leq \gamma_t \theta_i + (1-\gamma_t)\\
		& = \frac{1-\alpha_i}{1-\theta_i}(\theta_i-1) + 1\\
		& = \frac{\alpha_i-1}{\theta_i-1}(\theta_i-1) + 1\\
		& = \alpha_i.
	\end{align*}
	Similarly, in the latter case, it holds:
	\begin{align*}
		\gvec_i^\top\xvec_{t+1}&= 	\gvec_i^\top (\gamma_t\xvec^\diamond+ (1-\gamma_t)\widetilde{\xvec}_{t+1}) \\
		&= \gamma_t \theta_i + (1-\gamma_t)\gvec_i^\top\widetilde{\xvec}_{t+1}\\
		&\leq \gamma_t \theta_i + (1-\gamma_t)(\widehat\gvec_{t,i}+\boldsymbol{\beta}_t)^\top\widetilde{\xvec}_{t+1}\\
		& = \gamma_t(\theta_i-(\widehat\gvec_{t,i}+\boldsymbol{\beta}_t)^\top\widetilde{\xvec}_{t+1}) + (\widehat\gvec_{t,i}+\boldsymbol{\beta}_t)^\top\widetilde{\xvec}_{t+1}\\
		& = \frac{(\widehat{\gvec}_{t,i}+\boldsymbol{\beta}_t)^\top \widetilde\xvec_{t+1}-\alpha_i}{(\widehat{\gvec}_{t,i}+\boldsymbol{\beta}_t)^\top \widetilde\xvec_{t+1}-\theta_i} (\theta_i-(\widehat\gvec_{t,i}+\boldsymbol{\beta}_t)^\top\widetilde{\xvec}_{t+1}) + (\widehat\gvec_{t,i}+\boldsymbol{\beta}_t)^\top\widetilde{\xvec}_{t+1}\\
		& =  \frac{\alpha_i-(\widehat{\gvec}_{t,i}+\boldsymbol{\beta}_t)^\top \widetilde\xvec_{t+1}}{\theta_i-(\widehat{\gvec}_{t,i}+\boldsymbol{\beta}_t)^\top \widetilde\xvec_{t+1}} (\theta_i-(\widehat\gvec_{t,i}+\boldsymbol{\beta}_t)^\top\widetilde{\xvec}_{t+1}) + (\widehat\gvec_{t,i}+\boldsymbol{\beta}_t)^\top\widetilde{\xvec}_{t+1}\\
		& =\alpha_i.
	\end{align*}
	This concludes the proof.
\end{proof}

\subsection{Regret Lower Bound in MABs with Hard Constraints}
\label{proofs:lowerbound}
\lowerbound*
\begin{proof}
	%This proof follows similar steps as the proof of Theorem 3 of \cite{gerchinovitz2016refined}. 
	We split the proof in two parts: first, we prove a $\Omega\left(\sqrt{\omega T}+\Delta\sqrt{T}/\rho\right)$ lower bound for the expected regret of any randomized algorithm, when the losses are stochastic. In a stochastic setting, $\omega$ represents the double of the expected value of the loss of the best strategy, while $\Delta$ the double of the expected value of the difference between the strictly safe strategy and the benchmark. Second, we show that there exists at least a sequence of loss belonging to $\mathcal{B}_{\omega, \Delta, T}$ such that the lower bound from the first step holds.
\paragraph{Step 1}
	Let $B(\omega)$ indicate a Bernoulli probability distribution with mean $\omega \in (0,1)$.
	We start by introducing four instances of the hard constrained bandit problem where both losses and constraints costs are stochastic. To do so, we assume that both losses and constraints are sampled in advance, and introduce the following auxiliary sequences, for all $t \in [T]$:
	\begin{align*}
		W_t&\sim B(\omega), \\
		Y_t &\sim B(\omega+\psi), \\
		B_t &\sim B(1/2), \\
		C_t &\sim B(1/2-\rho), \\
		D_t &\sim B(1/2+\epsilon), \
	\end{align*}
	where $\omega \in (0,\frac{1}{2}),\psi \in (0, 1-\omega)$ and $\rho,\epsilon \in (0, \frac{1}{2})$.
	We consider four instances $\{\boldsymbol{\nu}_i\}_{i=1}^4$, each with $K=3$ actions, namely $a_1, a_2$ and $a_3$. In Table \ref{tab:instances_lb}, we summarize how losses and constraints costs are generated for each action and in each instance.
	\begin{table}[]
		\begin{center}
			\begin{tabular}{c|c|c|c|c|c|c}
				& $\ell_{t}(a_1)$ & $\ell_{t}(a_2)$ & $\ell_{t}(a_3)$  & $g_{t}(a_1)$ & $g_{t}(a_2)$ & $g_{t}(a_3)$ \\ \hline
				$\boldsymbol{\nu}_1$ & $\nicefrac{W_t}{2}$         & $\nicefrac{W_t}{2}$         & $\nicefrac{(W_t + \Delta )}{2}$ & $D_t$      & $D_t$      & $C_t$      \\
				$\boldsymbol{\nu}_2$ & $\nicefrac{W_t}{2}$        &$ \nicefrac{Y_t}{2}$         & $\nicefrac{(W_t+ \Delta)}{2} $ & $B_t$      & $B_t$      & $C_t$      \\
				$\boldsymbol{\nu}_3$ & $ \nicefrac{Y_t}{2}$         & $\nicefrac{W_t}{2}$         & $\nicefrac{(W_t+ \Delta)}{2} $ & $B_t$      & $B_t$      & $C_t$      \\
				$\boldsymbol{\nu}_4$ & $ \nicefrac{Y_t}{2}$         & $ \nicefrac{Y_t}{2}$         & $\nicefrac{(W_t+ \Delta)}{2} $ & $B_t$      & $B_t$      & $C_t$     
			\end{tabular}
			\caption{Summary of the losses (first three columns) and constraints costs (last three columns) associated to each of the four instances.}
			\label{tab:instances_lb}
		\end{center}
	\end{table}
	Instance $\boldsymbol{\nu}_1$ is the only one having different constraints costs, while action $a_3$ has, for every $t \in [T]$,  the same loss in every instance. In all instances, $\xvec^\diamond = (0,0,1)$ is the only strictly safe strategy.
	
	We start by considering $\boldsymbol{\nu}_1$: in order to be safe with high probability, for a given confidence level $\delta \in (0,1)$, any algorithm must satisfy:
	\begin{align*}
	\mathbb{P}_{\boldsymbol{\nu}_1}\left(\forall t \in [T] : x_t(a_3) \ge \frac{\epsilon}{\epsilon + \rho}\right) \ge 1-\delta,
	\end{align*}
	where $\xvec_t$ is the strategy of the algorithm at time $t$, and $\mathbb{P}_{\boldsymbol{\nu}_1}$ is the probability measure of instance $\boldsymbol{\nu}_1$ which encompasses the randomness of both environment and algorithm. As a consequence, we have:
	\begin{align*}
    \mathbb{P}_{\boldsymbol{\nu}_1}\left(\sum_{t=1}^T x_t(a_3) \ge T\frac{\epsilon}{\epsilon + \rho}\right) \ge 1-\delta.
	\end{align*}
	We now leverage Pinsker's inequality to relate the probability measures $\mathbb{P}_{\boldsymbol{\nu}_1}$ and $\mathbb{P}_{\boldsymbol{\nu}_j}$, with $j \in \{2,3,4\}$, as follows:
	\begin{align*}
    \mathbb{P}_{\boldsymbol{\nu}_j}\left(\sum_{t=1}^T x_t(a_3) \ge T\frac{\epsilon}{\epsilon + \rho}\right) \ge \mathbb{P}_{\boldsymbol{\nu}_1}\left(\sum_{t=1}^T x_t(a_3) \ge T\frac{\epsilon}{\epsilon + \rho}\right)- \sqrt{\frac{1}{2}\text{KL}_T(\boldsymbol{\nu}_j, \boldsymbol{\nu}_1)},
	\end{align*}
	where $\text{KL}_T(\boldsymbol{\nu}_j, \boldsymbol{\nu}_1)$ is the KL-divergence between the probability measures $\mathbb{P}_{\boldsymbol{\nu}_j}$ and $\mathbb{P}_{\boldsymbol{\nu}_1}$ after $T$ rounds of history.
	
	Using the KL decomposition argument from Lemma 1 of \citep{gerchinovitz2016refined} (which holds for correlated losses, as in our case), and by upper bounding the KL between two Bernoulli r.v.s using the $\chi^2$-divergence (see Lemma 2.8 from \cite{tsybakov}), for every $j \in \{2,3,4\}$, we have
	\begin{align*}
		\text{KL}_T(\boldsymbol{\nu}_j, \boldsymbol{\nu}_1) &\le T(2\text{KL}(B(\omega+\psi),B(\omega))+\text{KL}(B(1/2),B(1/2+\epsilon)) ) \notag \\
		& \le T\left(2\frac{\psi^2}{\omega(1-\omega)}+\frac{\epsilon^2}{(\frac{1}{2}+\epsilon)(\frac{1}{2}-\epsilon)}\right) \notag\\
		&\le \frac{1}{4}
	\end{align*}
	where the last step is obtained by setting $\psi = \frac{1}{4}\sqrt{\frac{\omega(1-\omega)}{T}}$ and $\epsilon = \frac{1}{6}\sqrt{\frac{1}{T}}$.
	Thus, for every $j \in \{2,3,4\}$:
    \begin{align*}
    \mathbb{P}_{\boldsymbol{\nu}_j}\left(\sum_{t=1}^T x_t(a_3) \ge T\frac{\epsilon}{\epsilon + \rho}\right) \ge \frac{3}{4}-\delta.
	\end{align*}
	Which implies, for every $\delta \in (0,\frac{1}{2})$:
	\begin{align}
	\label{eq:lb_06}
		\mathbb{E}_{\boldsymbol{\nu}_j}\left[ \sum_{t=1}^T x_t(a_3) \right]\ge \frac{1}{4}T\frac{\epsilon}{\epsilon + \rho}.
	\end{align}
	Now, we focus on instance $\boldsymbol{\nu}_2$ and $\boldsymbol{\nu}_3$: in the former the optimal strategy is to always pull $a_1$, while in the latter to always pull $a_2$. Hence, we can compute the expected regrets as:
	\begin{align*}
		&\begin{cases}
			2\mathbb{E}_{\boldsymbol{\nu}_2}[R_T] = \mathbb{E}_{\boldsymbol{\nu}_2}\left[\sum_{t=1}^T \left(\omega +\psi -\psi x_t(a_1)+ (\Delta-\psi)x_t(a_3)\right)\right]-T\omega,\\
			2\mathbb{E}_{\boldsymbol{\nu}_3}[R_T] = \mathbb{E}_{\boldsymbol{\nu}_3}\left[\sum_{t=1}^T \left(\omega +\psi -\psi x_t(a_2)+ (\Delta-\psi)x_t(a_3)\right)\right]-T\omega,
		\end{cases}\, \\
		&\begin{cases}
			2\mathbb{E}_{\boldsymbol{\nu}_2}[R_T] = T\psi -\psi \mathbb{E}_{\boldsymbol{\nu}_2}\left[\sum_{t=1}^T (x_t(a_1)+x_t(a_3))\right]+ \Delta \mathbb{E}_{\boldsymbol{\nu}_2}\left[\sum_{t=1}^T x_t(a_3)\right],\\
			2\mathbb{E}_{\boldsymbol{\nu}_3}[R_T] = T\psi -\psi \mathbb{E}_{\boldsymbol{\nu}_3}\left[\sum_{t=1}^T (x_t(a_2)+x_t(a_3))\right]+ \Delta \mathbb{E}_{\boldsymbol{\nu}_3}\left[\sum_{t=1}^T  x_t(a_3)\right].
		\end{cases}\,
	\end{align*}
	We now leverage Lemma A.1 from \citep{auer2002nonstochastic} and relate the expectations $\mathbb{E}_{\boldsymbol{\nu}_j}$, for $j\in\{2,3\}$, with $\mathbb{E}_{\boldsymbol{\nu}_4}$:
	\begin{align}
		\label{eq:lb_02}
			\mathbb{E}_{\boldsymbol{\nu}_j}\left[\sum_{t=1}^T x_t(a_i)\right] \le \mathbb{E}_{\boldsymbol{\nu}_4}\left[\sum_{t=1}^T x_t(a_i)\right]+ T\sqrt{\frac{\ln 2}{2}\text{KL}_T(\boldsymbol{\nu}_4, \boldsymbol{\nu}_j)},
	\end{align}
	for every $a_i$, where $\text{KL}_T$ indicates the KL-between the probability measures after $T$ rounds of history, and can be bounded as\footnote{Note that KL divergence is invariant to scaling, thus we directly compare Bernoulli distribution instead of the distributions derived by dividing by two.}:
	\begin{align}
		\label{eq:lb_03}KL_T(\boldsymbol{\nu}_4,\boldsymbol{\nu}_j) \le  T\frac{\psi^2}{\omega(1-\omega)}.
	\end{align}
	It follows:
	\begin{align*}
		&\begin{cases}
			2\mathbb{E}_{\boldsymbol{\nu}_2}[R_T] = T\psi -\psi \mathbb{E}_{\boldsymbol{\nu}_4}\left[\sum_{t=1}^T (x_t(a_1)+x_t(a_3))\right]- \sqrt{\frac{\ln 2}{2}\frac{\psi^4 T^3}{\omega(1-\omega)}}+ \Delta \mathbb{E}_{\boldsymbol{\nu}_2}\left[\sum_{t=1}^T x_t(a_3)\right],\\
			2\mathbb{E}_{\boldsymbol{\nu}_3}[R_T] = T\psi -\psi \mathbb{E}_{\boldsymbol{\nu}_4}\left[\sum_{t=1}^T (x_t(a_2)+x_t(a_3))\right]- \sqrt{\frac{\ln 2}{2}\frac{\psi^4 T^3}{\omega(1-\omega)}}+ \Delta \mathbb{E}_{\boldsymbol{\nu}_3}\left[\sum_{t=1}^T x_t(a_3)\right],
		\end{cases}\, \\
		&\begin{cases}
			2\mathbb{E}_{\boldsymbol{\nu}_2}[R_T] = \psi \mathbb{E}_{\boldsymbol{\nu}_4}\left[\sum_{t=1}^T x_t(a_2)\right]- \sqrt{\frac{\ln 2}{2}\frac{\psi^4 T^3}{\omega(1-\omega)}}+ \Delta \mathbb{E}_{\boldsymbol{\nu}_2}\left[\sum_{t=1}^T x_t(a_3)\right],\\
			2\mathbb{E}_{\boldsymbol{\nu}_3}[R_T] = \psi \mathbb{E}_{\boldsymbol{\nu}_4}\left[\sum_{t=1}^T x_t(a_1)\right]- \sqrt{\frac{\ln 2}{2}\frac{\psi^4 T^3}{\omega(1-\omega)}}+ \Delta \mathbb{E}_{\boldsymbol{\nu}_3}\left[\sum_{t=1}^T x_t(a_3)\right],
		\end{cases}\,
	\end{align*}
	where the first step is a consequence of Equation \eqref{eq:lb_03}, and the second step follows from the identity $x_t(a_1) + x_t(a_2)+x_t(a_3) = 1$, for every $t \in [T]$. We are now ready to lower bound the average expected regret between instances $\boldsymbol{\nu}_2$ and $\boldsymbol{\nu}_3$.
	\begin{align*}
		\mathbb{E}_{\boldsymbol{\nu}_2}[R_T]+\mathbb{E}_{\boldsymbol{\nu}_3}[R_T] &\ge \psi \mathbb{E}_{\boldsymbol{\nu}_4}\left[\sum_{t=1}^T(x_t(a_1)+x_t(a_2))\right]-\sqrt{\frac{\ln 2}{2}\frac{\psi^4 T^3}{\omega(1-\omega)}}+\Delta \mathbb{E}_{\boldsymbol{\nu}_3}\left[\sum_{t=1}^T x_t(a_3)\right] \\
		&=\frac{T\psi}{2} + \left(\Delta-\frac{\psi}{2}\right) \mathbb{E}_{\boldsymbol{\nu}_4}\left[\sum_{t=1}^T x_t(a_3)\right]-\sqrt{\frac{\ln 2}{2}\frac{\psi^4 T^3}{\omega(1-\omega)}} \\
		&\ge \frac{T\psi}{2} + \frac{T\Delta}{8}\frac{\epsilon}{\epsilon + \rho}-\sqrt{\frac{\ln 2}{2}\frac{\psi^4 T^3}{\omega(1-\omega)}} \\
		&\ge \frac{1}{16}\sqrt{\omega(1-\omega)T} + \frac{1}{48\rho}\Delta\sqrt{T}
	\end{align*}
	where the second inequality follows from Equation \eqref{eq:lb_06} and the fact that $\Delta \ge \psi$, the last inequality from the definition of $\psi$ and $\epsilon$, and the fact that  $\rho \ge \epsilon$. Noting that $\max\left\{\mathbb{E}_{\boldsymbol{\nu}_2}[R_T],\mathbb{E}_{\boldsymbol{\nu}_3}[R_T]\right\} \ge \left(\mathbb{E}_{\boldsymbol{\nu}_2}[R_T]+\mathbb{E}_{\boldsymbol{\nu}_3}[R_T]\right)/2$, we can conclude first step.
	\paragraph{Step 2} Note that $\frac{\Delta}{2} = \sqrt{\left(\frac{W_t+\Delta}{2}-\frac{W_t}{2}\right)^2}$, thus the (deterministic) term $\frac{\Delta}{2}\sqrt{T}$ is equivalent to the quadratic term describing the average squared distance between the optimal policy and the strictly safe strategy in instances $\boldsymbol{\nu}_2$ and $\boldsymbol{\nu}_3$. 

    Consider $T \ge \max\left\{2, (11+\ln T)\left(\frac{8}{3}\right)^2\right\}$, and $\frac{1}{2} \ge \omega \ge \frac{1}{T}\max\left\{1, \left(\frac{11}{2}+\ln T\right)\right\}$. Note that the set of existence of $\omega$ is never empty due to the condition on $T$.
    
	Consider modified versions of these instances with $\widetilde{\omega} = \frac{\omega}{2}$: we apply Bernstein Inequality to obtain that, with probability at least $1-\delta'$:
	\begin{align}
		\sum_{t=1}^T \lvec_t^\top\xvec^* &= \sum_{t=1}^T \frac{W_t}{2}  \notag \\
		&\le \frac{T\widetilde{\omega}}{2} + \frac{1}{2}\sqrt{2T\widetilde{\omega}(1-\widetilde{\omega})\ln\left(\frac{1}{\delta'}\right)}+\frac{1}{3}\ln\left(\frac{1}{\delta'}\right)  \notag \\
		&\le\frac{T\omega}{4}+\frac{T\omega}{4} \le \frac{T\omega}{2}, \label{eq:lb_08}
	\end{align}	
	by observing that $\left(\frac{8}{3}\right)^2\ln\left(\frac{1}{\delta'}\right) \le T\omega$, which holds true for $\delta' = \frac{1}{228T}$, given the conditions on $T$ and $\omega$. Thus, instances $\boldsymbol{\nu}_2$ and $\boldsymbol{\nu}_3$ may generate, with high probability, loss sequences in which the optimal one is bounded by $\frac{T\omega}{2}$. We conclude the proof by showing that, among these loss sequences, at least one satisfies the previously derived lower bound. Without loss of generality, consider $\boldsymbol{\nu}_2$ to be the instance with the higher expected regret, we then apply the lower bound derived in Step 1 with $\widetilde{\omega}$ to get
	\begin{align}
		\label{eq:lb_07}
		\mathbb{E}_{\boldsymbol{\nu}_2}[R_T(\lvec_{1:T})] \ge \frac{9}{1024}\sqrt{T\omega}+ \frac{1}{48\rho}\Delta\sqrt{T}, 
	\end{align}
	where $\lvec_{1:T}$ is a fixed loss sequence generated $\boldsymbol{\nu}_2$, and using that $\omega \le \frac{1}{2}$. Suppose by contradiction that, for every $\lvec_{1:T}$ s.t. Equation \eqref{eq:lb_08} holds, then
	\begin{align*}
		\mathbb{E}_{\boldsymbol{\nu}_2}[R_T(\lvec_{1:T})\mathbbm{1}_{Eq. \text{\eqref{eq:lb_08}}}] < \frac{9}{2048}\sqrt{T\omega}+ \frac{1}{48\rho}\Delta\sqrt{T}, 
	\end{align*}
	Then, by setting $\delta' = \frac{1}{228T}$, we have:
	\begin{align*}
		\mathbb{E}_{\boldsymbol{\nu}_2}[R_T(\lvec_{1:T})] &=  \	\mathbb{E}_{\boldsymbol{\nu}_2}[R_T(\lvec_{1:T})\mathbbm{1}_{Eq. \eqref{eq:lb_08}}] + \mathbb{E}_{\boldsymbol{\nu}_2}[R_T(\lvec_{1:T})(1-\mathbbm{1}_{Eq. \eqref{eq:lb_08}})] \\
		&< \frac{9}{2048}\sqrt{T\omega}+ \frac{1}{48\rho}\Delta\sqrt{T} + \frac{1}{228} \\
		&< \frac{9}{1024}\sqrt{T\omega}+ \frac{1}{48\rho}\Delta\sqrt{T},
	\end{align*}
	where the first inequality derives from bounding the regret with $T$ and the definition of $\delta'$, and the second inequality derives from observing that $T\omega > \left(\frac{2048}{9\cdot 228}\right)^2$, given the conditions on $T$ and $\omega$. This represents a contradiction with the lower bound derived in the first step, and thus there must exist a loss sequence such that Equation \eqref{eq:lb_08} and Equation \eqref{eq:lb_07} hold simultaneously.
\end{proof}

%%%%%%%%%%%%%%%%%%%%%%%%%%%%%%%%%%%%%%%%%%%%%%%%%%%%%%%%%%%%

\end{document}